\documentclass[12pt]{article}
 
\usepackage{myarticle}
\usepackage[colorinlistoftodos,prependcaption,textsize=tiny,backgroundcolor=black!5!white,bordercolor=red]{todonotes}
\usepackage{pgfplots}
\pgfplotsset{
  tick label style={font=\footnotesize},
  label style={font=\footnotesize},
  legend style={font=\footnotesize}
}

\newcommand{\R}{\mathbb R}             
\newcommand{\N}{\mathbb N}             

\newcommand{\st}{\textrm{s.t.} \ }          

\DeclareMathOperator*{\argmin}{argmin}
\DeclareMathOperator*{\argmax}{argmax}

\newcommand{\iter}[1]{^{(#1)}}

\let\phitmp\phi
\let\phi\varphi
\let\varphi\phitmp
\let\epstmp\eps
\let\eps\varepsilon
\let\varepsilon\epstmp

\newcommand{\set}[1]{\lbrace #1 \rbrace}                    

\newcommand{\seq}[2][]{(#2)_{#1}}                            
\newcommand{\norm}[2][]{\Vert #2 \Vert_{#1}}                 
\DeclarePairedDelimiter{\ceil}{\lceil}{\rceil}
\newcommand{\scal}[2]{\left\langle #1,#2 \right\rangle}      
\newcommand{\map}[3]{#1\colon #2 \to #3}                     
\newcommand{\opid}{I}                                        

\newcommand{\grad}[1][]{\nabla_{#1}}                              
\newcommand{\Hess}[1][]{\nabla^2_{#1}}                     

\renewcommand{\O}{\mathcal O}                          

\newcommand{\spLin}{\mathcal L}                       
\newcommand{\spSob}[2][]{\ensuremath{H^{#2\ifthenelse{\equal{#1}{\empty}}{}{,#1}}}}

\newcommand{\spVar}{\mathcal X}                         
\newcommand{\spPrm}{\mathcal U}                         

\newcommand{\x}{\bm x}                                  
\newcommand{\xz}{\x \iter{0}}                             
\newcommand{\xk}{\x \iter{k}}                             
\newcommand{\xkp}{\x \iter{k+1}}                          
\newcommand{\xkm}{\x \iter{k-1}}                          
\newcommand{\xK}{\x \iter{K}}                             
\newcommand{\xTp}{\x \iter{T^{\prime}}}                    

\newcommand{\z}{\bm z}                                  

\renewcommand{\u}{\bm u}                                

\newcommand{\w}{\bm w}                                  

\renewcommand{\b}{\bm b}                                

\renewcommand{\v}{\bm v}                                

\newcommand{\e}{\bm e}                                  

\renewcommand{\d}{\bm d}                                  

\newcommand{\JacVarFixMap}[1]{B_{#1}}
\newcommand{\JacVarFixMapk}{\JacVarFixMap{k}}

\newcommand{\JacVarFixMapkTp}{\JacVarFixMap{k+T^{\prime}}}
\newcommand{\JacVarFixMapmin}{\JacVarFixMap{}}
\newcommand{\JacPrmFixMap}[1]{C_{#1}}
\newcommand{\JacPrmFixMapk}{\JacPrmFixMap{k}}

\newcommand{\JacPrmFixMapkTp}{\JacPrmFixMap{k+T^{\prime}}}
\newcommand{\JacPrmFixMapmin}{\JacPrmFixMap{}}

\newcommand{\itrErr}[2][]{e_{#1}\iter{#2}}
\newcommand{\itrErrz}[1][]{\itrErr[#1]{0}}
\newcommand{\itrErrk}[1][]{\itrErr[#1]{k}}
\newcommand{\itrErrkp}[1][]{\itrErr[#1]{k+1}}

\newcommand{\fwdJac}[2][]{\dot X_{#1}\iter{#2}}
\newcommand{\fwdJacz}[1][]{\fwdJac[#1]{0}}
\newcommand{\fwdJack}[1][]{\fwdJac[#1]{k}}
\newcommand{\fwdJackp}[1][]{\fwdJac[#1]{k+1}}
\newcommand{\fwdJacK}[1][]{\fwdJac[#1]{K}}
\newcommand{\fwdJacErr}[2][]{\dot \e_{#1}\iter{#2}}
\newcommand{\fwdJacErrk}[1][]{\fwdJacErr[#1]{k}}
\newcommand{\fwdErr}[2][]{\dot e_{#1}\iter{#2}}
\newcommand{\fwdErrz}[1][]{\fwdErr[#1]{0}}
\newcommand{\fwdErrk}[1][]{\fwdErr[#1]{k}}
\newcommand{\fwdErrkp}[1][]{\fwdErr[#1]{k+1}}

\newcommand{\revJacVar}[2][K]{\bar X_{#1}\iter{#2}}

\newcommand{\revJacVark}[1][K]{\revJacVar[#1]{k}}
\newcommand{\revJacVarkp}[1][K]{\revJacVar[#1]{k+1}}
\newcommand{\revJacVarK}[1][K]{\revJacVar[#1]{K}}
\newcommand{\revJacPrm}[2][K]{\bar U_{#1}\iter{#2}}
\newcommand{\revJacPrmz}[1][K]{\revJacPrm[#1]{0}}
\newcommand{\revJacPrmk}[1][K]{\revJacPrm[#1]{k}}
\newcommand{\revJacPrmkp}[1][K]{\revJacPrm[#1]{k+1}}
\newcommand{\revJacPrmK}[1][K]{\revJacPrm[#1]{K}}
\newcommand{\revErr}[1]{\bar e\iter{#1}}
\newcommand{\revErrk}{\revErr{k}}

\newcommand{\argmapDom}{U}
\newcommand{\argmap}{\x^{\star}}

\newcommand{\fixMap}[1][]{\mathcal A_{#1}}

\newcommand{\deriv}[1][]{D_{#1}}
\newcommand{\derivx}{\deriv[\x]}
\newcommand{\derivz}{\deriv[\z]}
\newcommand{\derivu}{\deriv[\u]}

\newcommand{\appSecLabel}[1]{ssec:#1:proof}
\newcommand{\findProofIn}[1]{\begin{proof}The proof is in Section~\ref{\appSecLabel{#1}} in the appendix.\end{proof}}
\newcommand{\appSubSect}[2]{\subsection{Proof of {#1}~\ref{#2}} \label{\appSecLabel{#2}}}

\newcommand{\crefItms}[2]{%
  \hyperref[#2]{\namecref{#1}~\labelcref*{#1}\ref*{#2}}%
}

\graphicspath{{./figures/}}

\pgfplotsset{compat=1.18}

\begin{document}

\thispagestyle{empty}
\begin{center}
\vspace*{0.03\paperheight}
{\Large\bf Understanding the Curse of Unrolling \par}
\bigskip
\bigskip
{\large
Sheheryar Mehmood$^\dagger$ and Florian Knoll$^\star$ and Peter Ochs$^\dagger$ \\ \medskip
{\small
$^\dagger$~Saarland University, Saarbr\"{u}cken, Germany \\
$^\star$~Friedrich-Alexander University, Erlangen, Germany \\
}
}
\end{center}
\bigskip

\begin{abstract}
    Algorithm unrolling is ubiquitous in machine learning, particularly in hyperparameter optimization and meta-learning, where Jacobians of solution mappings are computed by differentiating through iterative algorithms. Although unrolling is known to yield asymptotically correct Jacobians under suitable conditions, recent work has shown that the derivative iterates may initially diverge from the true Jacobian, a phenomenon known as the curse of unrolling. In this work, we provide a non-asymptotic analysis that explains the origin of this behavior and identifies the algorithmic factors that govern it. We show that truncating early iterations of the derivative computation mitigates the curse while simultaneously reducing memory requirements. Finally, we demonstrate that warm-starting in bilevel optimization naturally induces an implicit form of truncation, providing a practical remedy. Our theoretical findings are supported by numerical experiments on representative examples.
\end{abstract}



\section{Introduction} \label{sec:intro}
Many problems in modern machine learning can be formulated as bilevel optimization tasks \citep{DKP+15, DZ20}, where an outer objective depends on the solution of an inner problem. That is, given a sufficiently smooth mapping $\map{\ell}{\spVar\times\spPrm}{\R}$, we solve
\begin{equation} \label{prob:outer:min}
  \min_{\u\in\spPrm} \ell (\argmap (\u), \u) \,,
\end{equation}
where, for each parameter vector $\u$, the quantity $\argmap (\u)$ is defined implicitly as the solution of an optimization problem or, more generally, as the solution of a nonlinear equation in the variable $\x$. Such formulations arise in a wide range of applications, including hyperparameter optimization \citep{Ben00, MDA15, Ped16}, meta-learning \citep{BHT+19, RFK+19, LMR+19}, implicit deep learning \citep{AK17, AAB+19, BKK19, BKK20}, and neural architecture search \citep{LSY19}.

In this work, we focus on deterministic settings in which $\argmap (\u)$ is characterized as the fixed point of a smooth parameterized mapping $\map{\fixMap}{\spVar\times\spPrm}{\spVar}$, that is, $\argmap (\u)$ solves
\begin{equation} \label{prob:non:lin:eq}
  \x = \fixMap (\x, \u) \,.
\end{equation}
for $\x\in\spVar$. This formulation serves as the fixed-point equation of many iterative algorithms used in practice, such as, gradient descent, to solve inner problems of the form
\begin{equation} \label{prob:inner:min}
  \min_{\x \in \spVar} f(\x, \u) \,.
\end{equation}
Under standard assumptions, the solution $\argmap (\u)$ of \eqref{prob:non:lin:eq} or \eqref{prob:inner:min} exists, is unique, and depends smoothly on $\u$.
\begin{figure}[t!]
  \centering
  \includegraphics[width=0.9\textwidth]{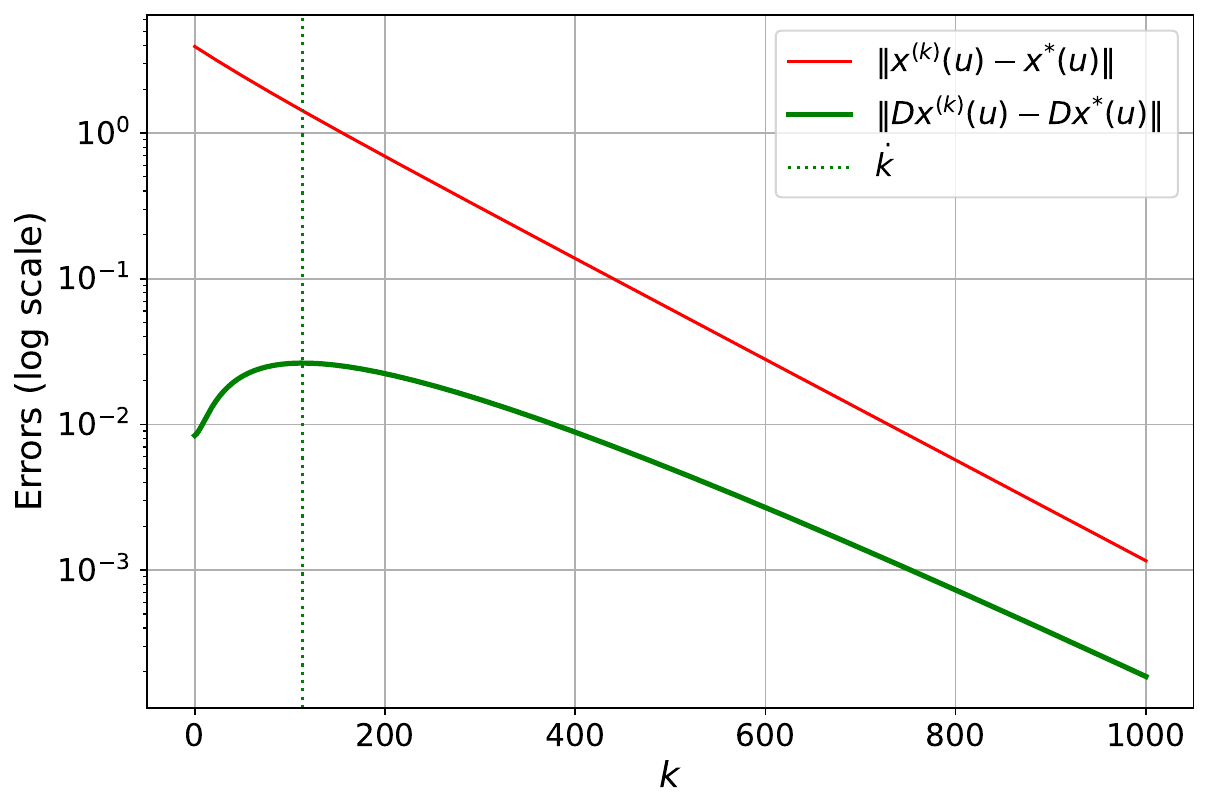}
  \caption{Iterate $\xk (\u)$ vs derivative $\deriv \xk (\u)$ error plot for gradient descent applied to $f (\x, u) \coloneqq \norm{A\x - \b}^2/2 + u\norm{\x}^2/2$. Unlike $\xk (\u)$, $\deriv\xk (\u)$ \emph{initially} drifts away from its limit before \emph{eventually} coming back to it. We provide a non-asymptotic understanding of this transient behavior, called the curse of unrolling, and study simple ways to mitigate it.}
  \label{fig:curse:intro}
\end{figure}

To solve the bilevel problem using gradient-based methods, we need to compute the gradient of the outer objective $\ell (\argmap (\u), \u)$ with respect to the parameter $\u$. By the chain rule, this requires the derivative of the solution map $\argmap(\u)$ with respect to $\u$. When $\argmap(\u)$ is defined implicitly as the fixed point of $\fixMap$, this derivative can be characterized using the implicit function theorem.

Under suitable regularity conditions (see Theorem~\ref{thm:BFPT:IFT}) on $\fixMap$ or the inner objective $f$, the map $\u \mapsto \argmap(\u)$ is differentiable and its Jacobian is given by
\begin{equation*}
  \deriv \argmap (\u) \coloneqq (\opid - \derivx\fixMap(\argmap(\u), \u))^{-1} \derivu\fixMap(\argmap(\u), \u) \,.
\end{equation*}
This approach, commonly referred to as implicit differentiation, has become a standard tool for computing gradients in bilevel optimization and related problems. In practice, since the exact fixed point $\argmap(\u)$ is rarely available, implicit differentiation is typically applied by replacing $\argmap(\u)$ with an approximate solution obtained after a finite number of iterations of an inner algorithm.

In practice, the solution map $\argmap(\u)$ is typically approximated by running an iterative algorithm associated with the mapping $\fixMap$. That is, for any parameter $\u$ and some initial point $\xz (\u)$, the algorithm generates a sequence of iterates $\seq[k\in\N]{\xk (\u)}$ through the updates
\begin{equation} \label{itr:FPI}
  \xkp (\u) \coloneqq \fixMap (\xk (\u), \u) \,,
\end{equation}
which converges to the fixed point $\argmap(\u)$ under suitable assumptions. After a finite number of iterations $K$, the algorithm output $\xK (\u)$ serves as an approximation of $\argmap(\u)$ in the outer objective.

Beyond approximating the solution itself, we may also differentiate the algorithmic map defining the iterates \citep{Dom12, DVFP14, OBP15}. Since $\fixMap$ is smooth, each iterate $\xk (\u)$ is differentiable with respect to $\u$, and its derivative can be computed recursively by differentiating through the update rule, that is,
\begin{equation}
  \deriv \xkp (\u) \coloneqq \JacVarFixMapk (\u) \deriv \xk (\u) + \JacPrmFixMapk (\u)
\end{equation}
where we define $\JacVarFixMapk (\u) \coloneqq \derivx \fixMap (\xk, \u)$ and $\JacPrmFixMapk (\u) \coloneqq \derivu \fixMap (\xk, \u)$. Running this recursion for $K$ iterations yields the Jacobian $\deriv \xK (\u)$ which provides an approximation of the true Jacobian $\deriv \argmap(\u)$. This process amounts to differentiating the finite-depth computational graph of $\u \mapsto \xK (\u)$ formed by the iterations of the algorithm and is, therefore, commonly referred to as algorithm unrolling.

A natural question is whether the favorable convergence properties of the algorithm iterates are reflected at the level of their derivatives. Under suitable assumptions, the sequence $\seq[k\in\N]{\xk(\u)}$ converges linearly to the fixed point $\argmap(\u)$, that is, the error $\norm{\xk(\u)-\argmap(\u)}$ decays at a geometric rate. One might therefore expect the derivative iterates $\deriv \xk(\u)$ to exhibit a similar behavior. Indeed, several works have established that, under standard regularity conditions, the sequence $\seq[k\in\N]{\deriv \xk(\u)}$ converges linearly to $\deriv \argmap(\u)$ with the same asymptotic rate \citep[see, for example,][]{Gil92, GBC+93, BKM+22, MO20, MO25}. However, the derivative error $\norm{\deriv \xk(\u)-\deriv \argmap(\u)}$ may initially increase with $k$ before eventually decreasing \citep{SGB+22}.

This initial increase in the derivative error, as illustrated in Figure~\ref{fig:curse:intro}, is referred to as the \emph{curse of unrolling} and is the subject matter of this paper.

\subsection{Contributions} \label{ssec:intro:contributions}
Our contributions can be summarized as follows:
\begin{enumerate}[label=(\roman*)]
  \item We analyze the non-asymptotic behavior of derivative iterates produced by algorithm unrolling and identify the factors responsible for the curse of unrolling.
  \item We study truncation through the same non-asymptotic lens and show explicitly why excluding early iterations from differentiation alleviates the curse.
  \item We show how warm-starting in bilevel optimization naturally induces truncation of derivative computations which avoids the need for explicit truncation.
  \item We provide empirical results that demonstrate the curse of unrolling and validate the predicted effects of truncation and warm-starting.
\end{enumerate}

\subsection{Related Work}
\textbf{Asymptotic Analysis of Unrolling:}
The convergence of the derivative sequence $\deriv\xk (\u)$ has been provided for general iterative processes \citep{Gil92,Bec94}, first-order methods for smooth \citep{MO20, PV23} and non-smooth optimization \citep{BKM+22, MO24, MO25}, and second-order methods \citep{GBC+93, BPV22}

\textbf{Curse of Unrolling:}
The first study on the curse phenomenon in unrolling was conducted by Scieur et al. \citep{SGB+22}. They provided comprehensive results for quadratic lower-level problems and identified that decreasing the step size reduces the effects of the curse at the cost of a slower convergence. They also designed an algorithm which when unrolled is free from the curse. However, their study was limited to the quadratic problems.

\textbf{Truncation in Unrolling:} Truncation was already hinted at by Gilbert \citep[see Section~4]{Gil92} but a comprehensive study was still missing. It was also studied to reduce the memory and compuational overhead of the backpropagation in \citep{SCHB19}. In \citep{BPV22}, the authors showed that differentiating only the last step of the algorithm is enough for super-linear algorithms and the same effect can be approximated in the linear algorithms by differentiating the last few steps. However, their study did not show how the truncation can help with the curse and in machine learning, the first order methods are predominantly used, for which, the optimal rates are only linear.

\subsection{Notation:}
In this paper, $\spVar$ and $\spPrm$ are Euclidean spaces with their inner products and the induced norms are represented by $\scal{\cdot}{\cdot}$ and $\norm{\cdot}$ respectively. The elements of $\spVar$ and $\spPrm$ are denoted by small bold letters, for instance, $\x\in\spVar$ and $\u\in\spPrm$. $\spLin (\spPrm, \spVar)$ denotes the space of linear maps from $\spPrm$ to $\spVar$ with induced (or spectral) norm topology, unless stated otherwise. A linear operator is denoted by a capital letter, for example, $Q\in\spLin (\spVar, \spVar)$ and $\rho (Q)$ denotes its spectral radius. Throughout this paper, we use standard dotted and barred notation for forward and reverse mode derivatives. Moreover, we frequently make use of \emph{forward pass} and \emph{forward mode} which are standard terminologies in Machine Learning and Automatic Differentiation literature respectively and must not be confused with each other.


\section{Preliminaries} \label{sec:preliminaries}

We recall the fixed-point formulation introduced in Section~\ref{sec:intro} and collect assumptions and results needed for the non-asymptotic analysis. The following global contractivity assumption helps to ensure existence, uniqueness, and smooth dependence of the fixed point on the parameter.
\begin{assumption} \label{ass:FixMap:Contraction}
  $\fixMap$ is $C^1$-smooth and for each $\u \in \spPrm$, $\fixMap(\cdot, \u)$ is a contraction on $\spVar$, that is, there exists $\rho(\u) \in [0, 1)$ such that $\fixMap(\cdot, \u)$ is $\rho(\u)$-Lipschitz continuous.
\end{assumption}
\begin{remark} \label{rem:FixMap:Contraction}
  \begin{enumerate}[label=(\roman*)]
    \item Under Assumption~\ref{ass:FixMap:Contraction}, the contraction modulus $\rho(\u)$ bounds the Jacobian $\derivx\fixMap (\x, \u)$ uniformly, that is, $\norm{\derivx\fixMap (\x, \u)} \leq \rho(\u)$ for all $(\x, \u) \in \spVar \times \spPrm$. This global uniform bound on the spectral norm can be replaced by a local one \citep[Chapter~2, Theorem~1]{Pol87} on the more general, spectral radius. We adopt it since it yields simpler non-asymptotic error bounds. For local results, it is typically sufficient to assume contractivity near a fixed-point.
    \item \label{itm:FixMap:GD:HB} Let $\map{f}{\spVar\times\spPrm}{\R}$ be a $C^2$-smooth function. For some $\alpha > 0$, we define the update map for gradient descent $\map{\fixMap[\alpha]}{\spVar\times\spPrm}{\spVar}$ by
    \begin{equation*}
      \fixMap[\alpha] (\x, \u) \coloneqq \x - \alpha \grad[\x] f (\x, \u) \,,
    \end{equation*}
    and for some $\beta\in[0, 1)$, we define the Heavy-ball update map $\map{\fixMap[\alpha, \beta]}{\spVar\times\spVar\times\spPrm}{\spVar\times\spVar}$ by
    \begin{equation*}
      \fixMap[\alpha, \beta] (\z, \u) \coloneqq ( \fixMap[\alpha] (\x_1, \u) + \beta (\x_1 - \x_2), \x_1 ) \,,
    \end{equation*}
    where $\z \coloneqq (\x_1, \x_2)$. Under global Lipschitz continuity of $\grad[\x] f$ and strong convexity of $f$ with respect to $\x$, $\fixMap[\alpha] (\cdot, \u)$ is a global contraction for suitable $\alpha$. On the other hand, $\fixMap[\alpha, \beta] (\cdot, \u)$
    does not satisfy Assumption~\ref{ass:FixMap:Contraction} in general and we can only have $\rho (\derivz \fixMap[\alpha, \beta] (\z, \u)) < 1$ \citep{Pol64, Pol87}.
  \end{enumerate}
\end{remark}
For a given parameter $\u$ and a smooth initialization map $\map{\xz}{\spPrm}{\spVar}$, we consider the fixed-point iteration \eqref{itr:FPI} for solving \eqref{prob:non:lin:eq}. Under Assumption~\ref{ass:FixMap:Contraction}, Banach fixed-point theorem \citep[Theorem~4.1.3]{AH05} ensures the existence and uniqueness of solution of \eqref{prob:non:lin:eq}, while the implicit function theorem \citep[Theorem~1B.1]{DR09} guarantees its smooth dependence on $\u$.
\begin{theorem} \label{thm:BFPT:IFT}
  Suppose $\fixMap$ satisfies Assumption~\ref{ass:FixMap:Contraction}. Then there exists a $C^1$-smooth map $\map{\argmap}{\spPrm}{\spVar}$ such that for all $\u \in \spPrm$, $\argmap(\u) = \fixMap(\argmap(\u), \u)$ is the unique fixed-point of $\fixMap(\cdot, \u)$ and its Jacobian is given by
  \begin{equation} \label{eq:IFT}
    \deriv \argmap (\u) \coloneqq (\opid - \JacVarFixMapmin(\u))^{-1} \JacPrmFixMapmin(\u) \,,
  \end{equation}
  where the maps $\map{\JacVarFixMapmin}{\argmapDom}{\spLin(\spVar, \spVar)}$ and $\map{\JacPrmFixMapmin}{\argmapDom}{\spLin(\spPrm, \spVar)}$ provide the partial Jacobians of $\fixMap$ evaluated at $(\argmap(\u), \u)$, that is,
  \begin{equation} \label{eq:IFT:FixMap:Derivatives}
    \begin{aligned}
      \JacVarFixMapmin (\u) &:= \derivx\fixMap(\argmap(\u), \u) \\
      \JacPrmFixMapmin (\u) &:= \derivu\fixMap(\argmap(\u), \u) \,.
    \end{aligned}
  \end{equation}
  Moreover, the sequence $\seq[k\in\N]{\xk (\u)}$ generated by \eqref{itr:FPI} converges linearly to $\argmap (\u)$ with rate $\rho (\u)$ from Assumption~\ref{ass:FixMap:Contraction}, that is, for all $k\in\N$,
  \begin{equation} \label{eq:conv:rate}
    \itrErrk (\u) \leq \rho(\u)^{k} \itrErrz (\u) \,,
  \end{equation}
  where we define $\itrErrk (\u) \coloneqq \norm{\xk (\u) - \argmap (\u)}$.
\end{theorem}

\subsection{Automatic Differentiation or Unrolling}
As suggested in Section~\ref{sec:intro}, an alternative to implicit differentiation for evaluating $\deriv\argmap (\u)$ is automatic differentiation which admits two modes: forward and reverse.

\textbf{Forward Mode AD:} The $C^1$-smoothness of $\fixMap$ allows us to write the following recursion.
\begin{equation}\label{itr:FPI:Forward:Mode}
  \fwdJackp (\u) := \JacVarFixMapk(\u) \fwdJack (\u) + \JacPrmFixMapk(\u) \,,
\end{equation}
where we set $\fwdJacz (\u)\coloneqq \deriv\xz(\u)$ and define sequence of maps $\map{\JacVarFixMapk}{\spPrm}{\spLin(\spVar, \spVar)}$ and $\map{\JacPrmFixMapk}{\spPrm}{\spLin(\spPrm, \spVar)}$ by:
\begin{equation} \label{eq:FPI:FixMap:Derivatives}
  \begin{aligned}
    \JacVarFixMapk (\u) &\coloneqq \derivx\fixMap(\xk(\u), \u) \\
    \JacPrmFixMapk (\u) &\coloneqq \derivu\fixMap(\xk (\u), \u) \,.
  \end{aligned}
\end{equation}
This recursion corresponds to forward mode automatic differentiation and is identical to the derivative update introduced in Section~\ref{sec:intro} so that $\deriv \xk (\u) = \fwdJack (\u)$. We define the forward mode errors by:
\begin{equation} \label{eq:error:fwd:AD:scalar}
  \fwdErrk (\u) \coloneqq \norm{\fwdJack (\u) - \deriv\argmap (\u)} \,.
\end{equation}
\textbf{Reverse Mode AD:} Let $\seq[k=0]{\xk (\u)}^{K-1}$ be fixed. We define the reverse-mode quantities $(\revJacVark, \revJacPrmk)$ backward for $k=K-1,\ldots,0$ via
\begin{equation} \label{eq:reverse-recursion}
  \begin{aligned}
    \revJacVark &\coloneqq \revJacVarkp \JacVarFixMapk (\u) \\
    \revJacPrmk &\coloneqq \revJacPrmkp + \revJacVarkp \JacPrmFixMapk (\u) \,,
  \end{aligned}
\end{equation}
where $\JacVarFixMapk$ and $\JacPrmFixMapk$ are the same as in \eqref{eq:FPI:FixMap:Derivatives}. By setting $\revJacPrmK \coloneqq 0$ and $\revJacVarK \coloneqq \opid$, the above recursion outputs the total derivative of the final iterate: $\u\mapsto\xK (\u)$, that is, $\deriv \xK (\u) = \revJacPrmz (\u)$. Following \eqref{eq:error:fwd:AD:scalar}, we define the reverse-mode error by:
\begin{equation} \label{eq:error:rev:AD:scalar}
  \revErrk (\u) \coloneqq \norm{\revJacPrmk (\u) - \deriv\argmap (\u)} \,,
\end{equation}
which should decrease for decreasing $k$.

\begin{table}[t!]
  \caption{Propagation of algorithm and derivative iterates.}
  \label{tab:unrolling}
  \begin{center}
    \renewcommand{\arraystretch}{1.5}
    \begin{small}
      \begin{tabular}{|c|c|c|}
        \hline
        \textbf{Iterations} & \textbf{Iterates} & \textbf{Propagation} \\
        \hline
        \textbf{Unrolled Algorithm} & $\seq[k=0]{\xk}^{K}$ & $0\to K$ \\ 
        \textbf{Forward Mode} & $\seq[k=0]{\fwdJack}^{K}$ & $0\to K$ \\ 
        \textbf{Reverse Mode} & $\seq[k=K]{\revJacPrmk}^{0}$ & $K\to 0$ \\ 
        \hline
      \end{tabular}
    \end{small}
  \end{center}
  \vskip -0.1in
\end{table}
\begin{remark} \label{rem:fwAD:bwAD:equivalence}
  The two modes propagate information in opposite directions, as summarized in Table~\ref{tab:unrolling}. Forward mode AD can be evaluated alongside the algorithm iterations, while reverse mode AD requires storing intermediate iterates during a forward pass before a backward sweep, which can be memory-intensive for large K. Moreover, reverse mode computes vector–Jacobian products, whereas forward mode computes Jacobian–vector products. For a better understanding and ease of comparison, we construct full Jacobians in both modes. Due to the equivalence of forward and reverse mode automatic differentiation \citep{GW08}, this yields $\deriv \xK = \fwdJacK = \revJacPrmz$.
\end{remark}

The following lemma provides a foundation for proving the convergence of the derivative sequence $\deriv \xk (\u)$.
\begin{lemma} \label{lem:AD:conv}
  Let $\seq[k\in\N]{\itrErrk}$ and $\seq[k\in\N]{\fwdErrk}$ be non-negative sequences. Assume that there exist constants $\rho\in[0,1)$ and $\Gamma \geq 0$ such that, for all $k\in\N$,
  \begin{equation} \label{eq:fwd:Jac:bound:k+1:k}
    \fwdErrkp \leq \rho \fwdErrk + \Gamma \itrErrk \,.
  \end{equation}
  Then, for any $k\in\N$, it holds that
  \begin{equation} \label{eq:fwd:Jac:bound:k:0}
    \fwdErrk \leq \rho^k \fwdErrz + \Gamma \sum_{i=0}^{k-1} \rho^{k - i - 1}\itrErr{i} \,.
  \end{equation}
  Moreover, when $\itrErrkp \leq \rho \itrErrk$, \eqref{eq:fwd:Jac:bound:k:0} reduces to
  \begin{equation} \label{eq:fwd:Jac:bound:k:0:final}
    \fwdErrk \leq \rho^k \fwdErrz + k\rho^{k-1} \Gamma \itrErrz \,.
  \end{equation}
\end{lemma}
\findProofIn{lem:AD:conv}
Below we assume global Lipschitz continuity of $\deriv\fixMap$ as well as a global bound on $\derivu\fixMap$ which along with the global contractivity in Assumption~\ref{ass:FixMap:Contraction} provides non-asymptotic error bound for the derivative sequence.
\begin{assumption} \label{ass:FixMap:Lipschitz:Derivative}
  For all $\u\in\spPrm$, there exist non-negative constants $M_{\x}(\u)$, $M_{\u}(\u)$ and $\kappa (\u)$ such that the mappings $\x \mapsto \derivx\fixMap(\x, \u)$ and $\x \mapsto \derivu\fixMap(\x, \u)$ are Lipschitz continuous with constants $M_{\x}(\u)$ and $M_{\u}(\u)$ respectively and $\norm{\derivu\fixMap(\x, \u)} \leq \kappa (\u)$ for all $\x \in \spVar$.
\end{assumption}
\begin{remark} \label{rem:FixMap:Lipschitz:Derivative}
  \begin{enumerate}[label=(\roman*)]
    \item Similarly, the local Lipschitz continuity and bound conditions are sufficient for asymptotic convergence results of $\deriv\xk (\u)$.
    \item \label{itm:Lipschitz:Derivative:GD:HB} For the setting of Remark~\ref{rem:FixMap:Contraction}\ref{itm:FixMap:GD:HB} and under additional global Lipschitz continuity of $\deriv(\grad[\x] f)$, the mappings $\deriv\fixMap[\alpha]$ and $\deriv\fixMap[\alpha, \beta]$ are Lipschitz continuous with constants proportional to $\alpha$ and the Lipschitz constant of $\deriv(\grad[\x] f)$.
  \end{enumerate}
\end{remark}
Using the above assumption we can provide the following bound which is useful for proving convergence of the derivative sequence $\fwdJack (\u)$.
\begin{lemma} \label{lem:FixMap:Derivative:Bound}
  Suppose $\fixMap$ satisfies Assumption~\ref{ass:FixMap:Lipschitz:Derivative} and the mappings $\JacVarFixMapk$, $\JacPrmFixMapk$, $\JacVarFixMapmin$ and $\JacPrmFixMapmin$ be defined in \eqref{eq:FPI:FixMap:Derivatives} and \eqref{eq:IFT:FixMap:Derivatives}. Then for any $\u \in \spVar$, we have the following bound.
  \begin{equation} \label{eq:FixMap:Derivative:Bound}
    \begin{aligned}
      \norm{\big( \JacVarFixMapk (\u) - \JacVarFixMapmin (\u) \big) &\deriv\argmap (\u) + \JacPrmFixMapk (\u) - \JacPrmFixMapmin (\u)} \\
      &\leq \Gamma (\u) \itrErrk (\u) \,,
    \end{aligned}
  \end{equation}
  where $\itrErrk (\u)$ is defined in Theorem~\ref{thm:BFPT:IFT} and $\Gamma (\u)$ is given by
  \begin{equation}
    \Gamma (\u) \coloneqq M_{\x} (\u) \frac{\kappa(\u)}{1 - \rho (\u)} + M_{\u} (\u) \,.
  \end{equation}
\end{lemma}
\findProofIn{lem:FixMap:Derivative:Bound}
The following result follows from Lemmas~\ref{lem:AD:conv} and \ref{lem:FixMap:Derivative:Bound} and provides non-asymptotic error bound for $\fwdJack (\u)$.
\begin{theorem}[Convergence of Derivative Iterates] \label{thm:AD:conv}
  Suppose $\fixMap$ satisfies Assumptions~\ref{ass:FixMap:Contraction} and \ref{ass:FixMap:Lipschitz:Derivative}. Then the sequence $\seq[k\in\N]{\fwdJack (\u)}$ generated by \eqref{itr:FPI:Forward:Mode} converges linearly to $\deriv\argmap (\u)$. In particular, for any $\u\in\spPrm$ and for all $k\in\N$,
  \begin{equation} \label{eq:AD:conv:rate}
    \fwdErrk (\u) \leq \rho(\u)^{k} \fwdErrz (\u) + k \rho(\u)^{k-1} \Gamma (\u) \itrErrz (\u) \,,
  \end{equation}
  where $\rho (\u)$ is the contraction constant from Assumption~\ref{ass:FixMap:Contraction} and $\Gamma (\u)$ is defined in Lemma~\ref{lem:FixMap:Derivative:Bound}.
\end{theorem}
\findProofIn{thm:AD:conv}


\section{The Curse of Unrolling} \label{sec:curse}

In this section, we study the finite-iteration behavior of the derivative iterates generated by algorithm unrolling. While asymptotic convergence of these iterates is well understood, we focus on their non-asymptotic behavior and show that, at finite depth, the derivative error may increase in the earlier iterations. Scieur et al. \citep{SGB+22} called this phenomenon, the curse of unrolling.

\begin{figure}[t]
  \begin{center}
    \centerline{\includegraphics[width=0.9\columnwidth]{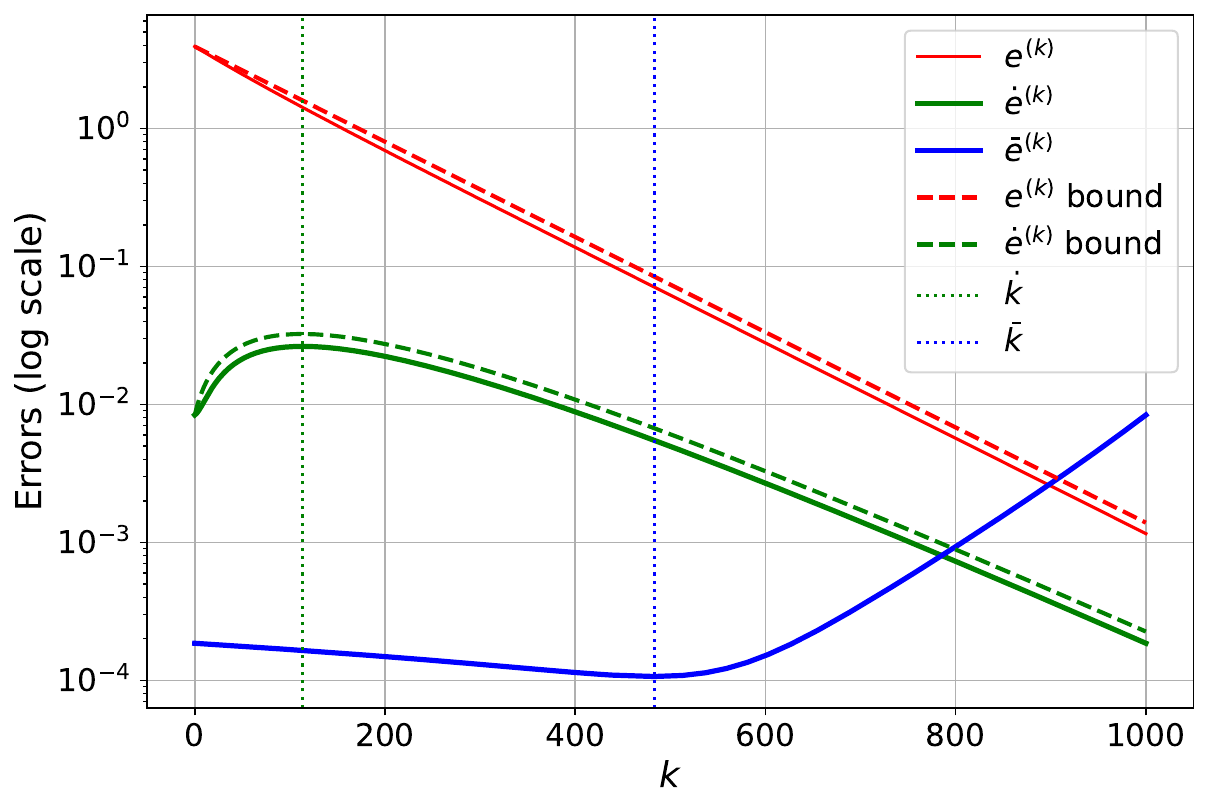}}
    \caption{Error evolution of $\itrErrk (\u)$, $\fwdErrk (\u)$, and $\revErrk (\u)$ generated by gradient descent applied to $f (\x, u) \coloneqq \norm{A\x - \b}^2/2 + u\norm{\x}^2/2$. The dashed lines denote the bounds given in \eqref{eq:conv:rate} and \eqref{eq:AD:conv:rate}. The vertical lines denote $\dot k$ and $\bar k$ defined in Sections~\ref{ssec:curse:forward} and \ref{ssec:curse:reverse} respectively.}
    \label{fig:curse:def}
  \end{center}
  \vskip -0.2in
\end{figure}
\subsection{Forward Mode AD} \label{ssec:curse:forward}
We refer to an initial increase in the forward-mode derivative error sequence $(\fwdJacErrk (\u))_{0\leq k\leq K}$ as the curse of unrolling. Although $\fwdJacErrk (\u)$ converges asymptotically, it may increase for a number of initial iterations before reaching its maximum. This behavior is clearly visible in Figure~\ref{fig:curse:def}, where the forward-mode derivative error initially increases before decaying. We define the index at which the forward-mode derivative error is largest as $\dot k := \argmax_{0 \leq k \leq K} \ \fwdJacErrk (\u)$.
\subsection{Reverse Mode AD} \label{ssec:curse:reverse}
A similar non-asymptotic behavior can be observed for the reverse mode derivative iterates. In contrast to forward mode, the reverse mode derivative error may initially decrease as the backward sweep begins, before increasing as the propagation continues toward earlier iterations as depicted in Figure~\ref{fig:curse:def}. We similarly define an index corresponding to the minimum value of the reverse-mode derivative error, that is, $\bar k := \argmin_{0 \leq k \leq K} \ \revErrk (u)$.
\subsection{Interpretation of the Non-asymptotic Results}
We now interpret the non-asymptotic bounds established in Section~\ref{sec:preliminaries} and explain how they give rise to the curse of unrolling. Lemma~\ref{lem:AD:conv} shows that the derivative error accumulates contributions of $\itrErrk$ from all previous iterations into a growing-then-decaying term $\O(k\rho^k)$. Lemma~\ref{lem:FixMap:Derivative:Bound} provides a bound on the size $\Gamma (\u)$ of these contributions. Finally, Theorem~\ref{thm:AD:conv} puts everything together and gives us the bound which contains a decaying term $\O(\rho (\u)^k)$, a growing-then-decaying term $\O(k \rho (\u)^k)$ which is the source of the curse and the algorithmic constants controlling the curse term. Since we only have an upper bound for the error, there is no guarantee that the curse of unrolling occurs. However, if it does, the upper bound allows us to control it. Therefore, understanding this upper bound is essential.

\subsection{Factors Governing the Curse}
We now briefly discuss the algorithmic constants in \eqref{eq:AD:conv:rate} that govern the curse of unrolling. This will help us build intuition on how to mitigate the curse later. We also demonstrate these effects empirically in Section~\ref{sec:experiments}.
\subsubsection{Rate of Convergence}
For faster algorithms, that is, when $\rho(\u)$ is small, the geometrically decaying term $\rho(\u)^k$ dominates the linear term $k$ more quickly, reducing both the magnitude and duration of the initial increase in the derivative error.
\subsubsection{Smoothness of the Update Map}
The smoothness of the update map, in particular, the Lip\-schitz constant of $\deriv\fixMap$ plays a critical role. A larger Lip\-schitz constant delays the decline of the curse term $k \rho (\u)^k$, leading to a more pronounced curse of unrolling. The Lip\-schitz constant is mainly affected by two factors:

\textbf{Small step size:} For algorithms such as gradient descent or the Heavy-ball method, a smaller step size leads to a small Lipschitz constant (see Remark~\ref{rem:FixMap:Lipschitz:Derivative}\ref{itm:Lipschitz:Derivative:GD:HB}).

\textbf{Almost-Linear Objectives:} When the objective function is nearly linear, for example in the case of Huber loss \citep{Hub92}, the quantities $\Hess[\x] f$ and $\derivu \grad[\x] f$ are small, resulting in a milder curse of
unrolling.

\subsubsection{Quality of Initialization}
Finally, the quality of the initial iterate $\xz(\u)$ also affects the severity of the curse. A larger initial error $\itrErrz(\u)$ allows the curse term $k \rho(\u)^k$ to grow to a larger magnitude before eventually decreasing, resulting in a more significant initial increase in the derivative sequence.


\section{Mitigating the Curse of Unrolling} \label{sec:truncation}

In Section~\ref{sec:curse}, we identified the factors that give rise to the curse of unrolling. Most of these factors are algorithmic constants, such as the convergence rate $\rho$ or the Lipschitz constant of $\deriv\fixMap$, which can only be improved through parameter tuning (e.g., step size) or by changing the algorithm. The only factor that can be influenced without altering the algorithm is the quality of the initial iterate $\xz (\u)$. In this section, we study a simple strategy to mitigate the curse by truncating early iterations from the derivative computation.

\subsection{Truncation or Late-start}
Intuitively, the curse is most pronounced during the early iterations, where the non-asymptot\-ic derivative error may grow. Excluding these iterations from differentiation helps avoid this unfavorable regime. Accordingly, we truncate the derivative computation by excluding early algorithm iterations. In reverse mode, truncation amounts to stopping the backward sweep at an index $T^{\prime}$ instead of propagating derivatives all the way to the initial iterate. Similarly, in forward mode, truncation is implemented by starting the derivative recursion at the same index $T^{\prime}$, rather than at the beginning of the algorithm.

More formally, in forward mode, we allow the first $T^{\prime}\in\N$ algorithm iterations to run idly and start the derivative recursion from $\xTp$. Specifically, we compute
\begin{equation} \label{itr:FPI:Forward:Mode:Late:Start}
  \fwdJackp[T^{\prime}] (\u) := \JacVarFixMapkTp (\u) \fwdJack[T^{\prime}] (\u) + \JacPrmFixMapkTp (\u) \,,
\end{equation}
for $k = 0,1,\dots$, where we use the same initialization $\fwdJacz[T^{\prime}] := \deriv\xz(\u)$ as before for simplicity. Here, the subscript $T^{\prime}$ indicates the late-started derivative sequence, which differs from the vanilla forward-mode sequence $\fwdJack (\u) = \deriv\xk(\u)$. If we run $K^{\prime}$ algorithm iterations in total with the initial $T^{\prime}$ idle iterations, then the output of the late-started forward mode will be $\fwdJac[T^{\prime}]{K^{\prime}-T^{\prime}}$.

For reverse mode, we first run the algorithm for $K^{\prime}$ iterations and then stop the backward sweep at $k=T^{\prime}$ instead of propagating derivatives back to $k=0$. The output of the truncated reverse mode is therefore $\revJacPrm[K^{\prime}]{T^{\prime}}$ rather than $\revJacPrmz[K^{\prime}]$.

Table~\ref{tab:Unrolling:Truncation} summarizes the algorithm and derivative iterations. Since both modes evaluate the same product arising from the chain rule, the late-started forward mode and the truncated reverse mode satisfy $\fwdJac[T^{\prime}]{K^{\prime}} = \revJacPrm[K^{\prime}]{T^{\prime}}$.

\subsection{Convergence with Truncation}
We now theoretically justify why truncation, or late-start, alleviates the curse of unrolling. For clarity, we focus on forward mode. The same conclusions apply to reverse mode by equivalence (see Remark~\ref{rem:fwAD:bwAD:equivalence}).
\begin{theorem}[Convergence with Truncation] \label{thm:AD:conv:Late:Start}
  Suppose $\fixMap$ satisfies Assumptions~\ref{ass:FixMap:Contraction} and \ref{ass:FixMap:Lipschitz:Derivative}. Then, for any late-start index $T^{\prime}\in\N$, the sequence $\seq[k\in\N]{\fwdJack[T^{\prime}] (\u)}$ generated by \eqref{itr:FPI:Forward:Mode:Late:Start} converges linearly to $\deriv\argmap (\u)$. In particular, for any $\u\in\spPrm$ and for all $k\in\N$,
  \begin{equation} \label{eq:AD:conv:rate:Late:Start}
    \fwdErrk[T^{\prime}] (\u) \leq \rho(\u)^{k} \fwdErrz (\u) + k \rho(\u)^{k+T^{\prime}-1} \Gamma (\u) \itrErrz (\u) \,,
  \end{equation}
  where $\rho (\u)$ is the contraction constant from Assumption~\ref{ass:FixMap:Contraction} and $\Gamma (\u)$ is defined in Lemma~\ref{lem:FixMap:Derivative:Bound}.
\end{theorem}
\findProofIn{thm:AD:conv:Late:Start}
\begin{remark} \label{rem:AD:conv:Late:Start}
  \begin{enumerate}[label=(\roman*)]
    \item When $T^{\prime}$ is sufficiently large, the curse term is multiplied by an additional factor $\rho(\u)^{T^{\prime}}$. This significantly attenuates the contribution of the growing-then-decaying term in the non-asymptotic bound and thereby alleviates the curse.
    \item The effectiveness of truncation is particularly easy to visualize in reverse mode AD. As shown in Figure~\ref{fig:curse:def}, the iterate $\revJacPrmk (\u)$ is closest to $\deriv \argmap (\u)$ at $k = \bar k$. It is therefore more reasonable to use $\revJacPrm{\bar k} (\u)$ as an estimate of $\deriv\argmap (\u)$ rather than $\revJacPrmz (\u)$.
    \item If Assumptions~\ref{ass:FixMap:Contraction} and \ref{ass:FixMap:Lipschitz:Derivative} are replaced by their local counterparts, the derivative iterates $\fwdJack (\u)$ and $\fwdJack[T] (\u)$ still converge \citep[Section~1.4.2]{MO24} but the linear error bounds provided in Theorems~\ref{thm:AD:conv} and \ref{thm:AD:conv:Late:Start} hold only asymptotically. In this regime, truncation becomes even more relevant, as it avoids the unpredictable behavior of the earlier algorithm iterates.
  \end{enumerate}
\end{remark}

\begin{table}[t!]
  \caption{Propagation of the iterates under truncation.}
  \label{tab:Unrolling:Truncation}
  \begin{center}
    \renewcommand{\arraystretch}{1.5}
    \begin{small}
      \begin{tabular}{|c|c|c|}
        \hline
        \textbf{Iterations} & \textbf{Iterates} & \textbf{Propagation} \\
        \hline
        \textbf{Idle Algorithm} & $\seq[k=0]{\xk}^{T^{\prime}}$ & $0 \to T^{\prime}$ \\
        \textbf{Unrolled Algorithm} & $\seq[k=T^{\prime}]{\xk}^{K^{\prime}}$ & $T^{\prime} \to K^{\prime}$ \\
        \textbf{Forward Mode} & $\seq[k=T^{\prime}]{\fwdJack}^{K^{\prime}}$ & $T^{\prime} \to K^{\prime}$ \\ 
        \textbf{Reverse Mode} & $\seq[k=K^{\prime}]{\revJacPrmk}^{T^{\prime}}$ & $K^{\prime} \to T^{\prime}$ \\ 
        \hline
      \end{tabular}
    \end{small}
  \end{center}
\end{table}

\subsection{Reallocation of Computational Resources} \label{ssec:truncation:realloc}

When truncation is employed, we must decide how to allocate computational cost between running the algorithm itself and differentiating through its iterations. We may, for instance, fix the total number of algorithm iterations $K^{\prime}$ and vary the truncation index $T^{\prime}$, or alternatively fix the number of derivative iterations $K^{\prime} - T^{\prime}$. Both viewpoints are reasonable, but neither directly reflects practical constraints.

Instead, we adopt a computational budget perspective and assume that the combined cost of evaluating the algorithm and its derivatives is fixed and must not be exceeded. When there is no truncation, we run both the algorithm and derivative sweeps for some $K\in\N$ iterations. However, the computational cost of evaluating the Jacobian-vector product (resp. vector-Jacobian product) using forward (resp. reverse) mode AD is $\omega$ times that of evaluation of the function itself where $\omega \in [2, 2.5]$ for forward mode and $\omega \in [3, 4]$ for reverse mode \citep[Chapter~3]{GW08}. In total, the computational cost without truncation is proportional to $K + \omega K$ which will be our computational budget. Therefore, if we reduce the number of derivative iterations by some $T\in\N$, we can increase the number of algorithm iterations by $\omega T$ without exceeding our computational budget. Our new cost is still $(K + \omega T) + \omega (K - T) = K + \omega K$. Therefore, we run our algorithm for a total $K^{\prime} \coloneqq K + \omega T$ number of iterations. Out of these, the last $K - T$ are used in the derivative computation while the remaining $T^{\prime} \coloneqq (K + \omega T) - (K - T) = T + \omega T$ algorithm iterations are run idly. That is, the iterates from $\xz$ to $\x\iter{T^{\prime}-1}$ are not added to the computational graph and only those from $\xTp$ to $\x\iter{K + \omega T - 1}$ are used in the derivative computation process. The late-started forward mode AD iterates $\fwdJack[T^{\prime}]$ are generated by using \eqref{itr:FPI:Forward:Mode:Late:Start} for $k = 0,\ldots, K - T$.
\begin{example}
  Suppose that, without truncation, both the algorithm and the derivative recursions are run for $K = 500$ iterations. If we reduce the number of derivative iterations by $20\%$ and assume $\omega = 3$, then the truncation index is $T = 100$, since the number of derivative iterations becomes $K - T = 0.8K = 400$. Under a fixed computational budget, this reduction allows the number of algorithm iterations to be increased to $K^{\prime} = K + \omega T = 800$.
\end{example}

\subsection{Optimizing the Truncation Index}
By using the resource allocation strategy from the previous section, we now try to find a truncation index that not only reduces the effect of the curse but also leaves sufficient iterations while keeping the computational cost fixed. For finding the optimal truncation index $T$ in $\set{0, \ldots, K}$, we minimize the upper bound provided in Theorem~\ref{thm:AD:conv:Late:Start} for $k = K - T$, that is,
\begin{equation} \label{prob:optimal:T}
  \min_{0 \leq T\leq K} \rho^{K - T} \fwdErrz + (K - T) \rho^{K + \omega T - 1} \Gamma \itrErrz \,.
\end{equation}
Here, we drop the dependence on $\u$ for brevity. If we relax the optimization space to $[0, K]$, then the problem is (strictly) convex as established in the following lemma.
\begin{lemma} \label{lem:optimal:T:strictly:convex:problem}
  Suppose $\fixMap$ satisfies Assumptions~\ref{ass:FixMap:Contraction} and \ref{ass:FixMap:Lipschitz:Derivative} where $\rho \in (0, 1)$. Then the optimization problem given in \eqref{prob:optimal:T}, relaxed to $T\in[0, K]$, is convex. Moreover, the problem is strictly convex provided that either $\fwdErrz > 0$, or $\itrErrz > 0$ and $\Gamma > 0$.
\end{lemma}
\findProofIn{lem:optimal:T:strictly:convex:problem}
\begin{remark}
  When $K$ is not large and the constants are known, the discrete problem in \eqref{prob:optimal:T} can be solved efficiently. For large $K$, we may instead solve the relaxed problem in Lemma~\ref{lem:optimal:T:strictly:convex:problem} to obtain an approximation of optimal $T$.
\end{remark}

\subsection{Limitations of Optimizing the Truncation Index}
Although, the procedure for choosing an optimal truncation index by solving \eqref{prob:optimal:T} provides clear theoretical insights on the choice from a practical perspective, it has several limitations. First, it optimizes an upper bound on the derivative error rather than the error itself. Moreover, key quantities such as the convergence rate and Lipschitz constants are typically unknown. Furthermore, the non-asymptotic bounds rely on global contraction and smoothness assumptions that may not hold in practice. Finally, modern automatic differentiation frameworks such as PyTorch \citep{PGM+19}, TensorFlow \citep{ABC+16}, and JAX \citep{BFH+18} require static computational graphs, which prevents adaptive selection of the truncation index when the termination of the inner solver is influenced by a stopping criterion.

Nonetheless, choosing a truncation index using heuristics \citep{SCHB19} and truncating the backward pass is still a practical and effective strategy. In addition to alleviating the curse of unrolling, truncation reduces memory overhead and allows computational resources to be reallocated to the forward pass.


\section{Warm-Starting \& Implicit Truncation} \label{sec:warm:start}
In the previous section, we studied explicit truncation of the derivative computation as a principled way to mitigate the curse of unrolling. While this analysis provides valuable insight, explicit truncation requires selecting a truncation index and relies on quantities that are often unknown in practice. In this section, we discuss warm-starting, a widely used strategy that naturally induces an implicit form of truncation without requiring such choices.

\subsection{Warm-starting in Sequence of Related Problems}
Warm-starting is a standard technique for solving sequences of closely related optimization problems, where the solution of a previous problem is reused as the initial point for the next. Suppose we aim to solve \eqref{prob:non:lin:eq} or \eqref{prob:inner:min} for a sequence of slowly changing parameters $\seq[r=0]{\u\iter{r}}^{R}$. Then after solving the problem associated with $\u\iter{r}$ through \eqref{itr:FPI}, for $r = 0, \ldots, R-1$, the resulting approximation for $\argmap (\u\iter{r})$ can be used to initialize $\xz(\u\iter{r+1})$ for solving the next problem. This strategy has been shown to reduce the number of iterations required to reach a given accuracy and is widely used in applications such as video and image processing, time-varying inverse problems \citep{HR22}, and continuation methods \citep{AG03}, where consecutive problems differ only slightly.

From this perspective, bilevel optimization naturally gives rise to a sequence of related inner problems indexed by the outer variable. As the outer iterate evolves, the corresponding inner problems typically change smoothly, making warm-starting a natural and effective choice. More details can be found in Section~\ref{sec:bilevel:app} in the appendix.

\subsection{Warm-starting as Implicit Truncation}
From the viewpoint of algorithm unrolling, warm-starting has an important and previously underappreciated effect. By initializing the inner algorithm close to its fixed point, warm-starting effectively bypasses the early iterations where the curse of unrolling is most pronounced. As a result, the derivative computation avoids the unfavorable transient regime identified in Section~\ref{sec:curse}.

Warm-starting, therefore, induces an implicit form of truncation: early algorithm iterations are either skipped or substantially shortened, without explicitly modifying the computational graph or choosing a truncation index. Unlike explicit truncation, this mechanism is adaptive, compatible with automatic differentiation frameworks, and already present in many practical implementations of bilevel optimization.

\section{Experiments} \label{sec:experiments}

\section{Conclusion}
\label{sec:Conc}
We studied the non-asymptotic behavior of algorithm unrolling for differentiating solution maps of parametric optimization problems. Despite asymptotic convergence guarantees, derivative iterates may exhibit an initial transient growth or the curse of unrolling, even when the underlying algorithm converges monotonically. Our analysis identifies the factors governing this behavior and shows how truncating early iterations effectively mitigates the curse. We further demonstrate that warm-starting, a standard practice in bilevel optimization, naturally induces an implicit form of truncation. Numerical experiments support our theoretical findings and highlight the practical relevance of truncation-based strategies for unrolling-based differentiation.


\appendix
\section{Proofs}
\appSubSect{Lemma}{lem:AD:conv}
\begin{proof}
  The bound \eqref{eq:fwd:Jac:bound:k:0} follows by iterating \eqref{eq:fwd:Jac:bound:k+1:k} and a straightforward induction. The bound \eqref{eq:fwd:Jac:bound:k:0:final} then follows by substituting $\itrErr{i} \leq \rho^i \itrErrz$ --- which is obtained by recursively expanding the inequality $\itrErrkp \leq \rho \itrErrk$ --- into \eqref{eq:fwd:Jac:bound:k:0} which makes the summand $\rho^{k-1}\itrErrz$.
\end{proof}

\appSubSect{Lemma}{lem:FixMap:Derivative:Bound}
\begin{proof}
  By expanding $\norm{ ( \JacVarFixMapk (\u) - \JacVarFixMapmin (\u) ) \deriv\argmap (\u) + \JacPrmFixMapk (\u) - \JacPrmFixMapmin (\u) }$ and subsituting $\norm{\deriv\argmap} \leq \kappa(\u) / (1 - \rho (\u))$ \citep[Theorem~2.2]{Chr94}, we arrive at our desired result.
\end{proof}

\appSubSect{Theorem}{thm:AD:conv}
\begin{proof}
  First we rewrite $\fwdJackp (\u) - \deriv\argmap (\u)$ as:
  \begin{equation*}
    \begin{aligned}
      \fwdJackp (\u) &- \deriv\argmap (\u) = \JacVarFixMapk (\u) \fwdJack (\u) + \JacPrmFixMapk - \JacVarFixMapmin (\u) \deriv\argmap (\u) - \JacPrmFixMapmin (\u) \\
      &= \JacVarFixMapk (\u) \Big(\fwdJack (\u) - \deriv\argmap (\u)\Big) + \Big( \big( \JacVarFixMapk (\u) - \JacVarFixMapmin (\u) \big) \deriv\argmap (\u) + \JacPrmFixMapk (\u) - \JacPrmFixMapmin (\u) \Big) \,.
    \end{aligned}
  \end{equation*}
  Since the two terms in the last expression are bounded by $\rho (\u) \fwdErrk (\u)$ and $\Gamma (\u) \itrErrk (\u)$ (from Lemma~\ref{lem:FixMap:Derivative:Bound}) respectively, we obtain
  \begin{equation*}
    \fwdErrkp (\u) \leq \rho (\u) \fwdErrk (\u) + \Gamma (\u) \itrErrk (\u) \,.
  \end{equation*}
  The error bound in \eqref{eq:AD:conv:rate} is then obtained by applying Lemma~\ref{lem:AD:conv}.
\end{proof}

\appSubSect{Theorem}{thm:AD:conv:Late:Start}
\begin{proof}
  Applying the same strategy as in the proof of Theorem~\ref{thm:AD:conv}, we write
  \begin{equation*}
    \begin{aligned}
        \fwdJackp[T] (\u) - \deriv\argmap (\u) &= \JacVarFixMap{k+T} (\u) \Big(\fwdJack[T] (\u) - \deriv\argmap (\u)\Big) \\
        &+ \Big( \big( \JacVarFixMap{k+T} (\u) - \JacVarFixMapmin (\u) \big) \deriv\argmap (\u) + \JacPrmFixMap{k+T} (\u) - \JacPrmFixMapmin (\u) \Big) \,.
    \end{aligned}
  \end{equation*}
  Because, the two terms in the expression on the right side are bounded by $\rho (\u) \fwdErrk[T] (\u)$ and $\Gamma (\u) \itrErr{k+T} (\u)$ (from Lemma~\ref{lem:FixMap:Derivative:Bound}) respectively, and by defining the iterate error at $k+T$ by $\itrErr[T]{k} \coloneqq \itrErr{k+T}$, we obtain the following inequality
  \begin{equation*}
    \fwdErrkp[T] (\u) \leq \rho (\u) \fwdErrk[T] (\u) + \Gamma (\u) \itrErrk[T] (\u) \,.
  \end{equation*}
  Applying Lemma~\ref{lem:AD:conv}, we obtain
  \begin{equation*}
    \fwdErrk[T] (\u) \leq \rho(\u)^{k} \fwdErrz[T] (\u) + k \rho(\u)^{k-1} \Gamma (\u) \itrErrz[T] (\u) \,,
  \end{equation*}
  Since $\itrErrz[T] = \itrErr{T} \leq \rho(\u)^T \itrErrz (\u)$ and $\fwdErrz[T^{\prime}] = \fwdErrz$, we arrive at our desired result.
\end{proof}

\appSubSect{Lemma}{lem:optimal:T:strictly:convex:problem}
\begin{proof}
  We first rewrite $h$ by defining $A\coloneqq \rho^{K} \fwdErrz \geq 0$, $B\coloneqq \rho^{K-1}\Gamma\itrErrz \geq 0$ and $\delta \coloneqq -\log (\rho) > 0$ since $\rho \in (0, 1)$. Here $\log$ denotes the natural logarithm. Substituting $A$, $B$, and $\delta$, we get
  \begin{equation*}
    h (T) = A e^{\delta T} + B (K - T) e^{-\omega\delta T} \,.
  \end{equation*}
  We now compute the second derivatives of $h$ and check for its sign. For $h$ to be convex $h^{\prime\prime}$ must not be negative. Moreover, $h^{\prime\prime} > 0$ is a sufficient condition for the strict convexity of $h$. Now the first derivative is given by:
  \begin{equation*}
    h^{\prime} (T) = A\delta e^{\delta T} + B e^{-\omega\delta T} (-1 -\omega\delta(K - T)) \,.
  \end{equation*} 
  Similarly, the second derivative reads:
  \begin{equation*}
    h^{\prime\prime} (T) = A\delta^2 e^{\delta T} + B \omega\delta e^{-\omega\delta T} (2 + \omega \delta(K - T)) \,.
  \end{equation*}
  Since all constants are non-negative and $T \in [0, K]$, therefore $h^{\prime\prime} \geq 0$. Moreover, $h^{\prime\prime}$ is positive when either $\fwdErrz$, or both $\itrErrz$ and $\Gamma$ are positive. This concludes our proof.
\end{proof}

\section{Warm-Starting in Bilevel Optimization} \label{sec:bilevel:app}
Warm-starting is a widely used strategy in bilevel optimization, where the solution of the inner problem from the previous outer iteration is used as the initial point for the estimation of the next inner solution. Both theoretical and empirical studies have shown that warm-starting significantly reduces the number of inner iterations required to achieve a given accuracy \citep{JYL21}.

Let $\map{\ell}{\spVar}{\R}$ be $C^1$-smooth and $\map{\fixMap}{\spVar\times\spPrm}{\spVar}$ satisfies Assumptions~\ref{ass:FixMap:Contraction} and \ref{ass:FixMap:Lipschitz:Derivative}. We consider the bilevel optimization problem
\begin{equation} \label{prob:bilevel}
  \begin{aligned}
    \min_{\u \in \spPrm} \quad & \ell (\argmap (\u)) \\
    \st \quad & \argmap (\u) = \fixMap (\argmap (\u), \u) \,.
  \end{aligned}
\end{equation}
For simplicity we assume that both problems are deterministic. Algorithm~\ref{alg:bilevel} summarizes a standard bilevel optimization procedure that uses reverse-mode automatic differentiation to compute the hypergradients and utilizes warm-starting. In particular, at each outer iteration $r$, the lower-level problem is approximately solved using an iterative algorithm. Once the algorithm terminates, we use the reverse mode AD to compute the gradient of the outer objective $\grad (\ell \circ \x\iter{K_r}) (\u\iter{r})$, which is referred to as hypergradient or the metagradient in Machine Learning literature. After updating the outer variable using a gradient descent step, we warm-start the inner algorithm using the solution from the previous outer iteration.

\begin{Algorithm}[Bilevel Optimization with Warm-Starting]
  \label{alg:bilevel}
  \begin{algorithmic}
    \STATE {\bfseries Input:} \\ \qquad $R\in\N$, total outer iterations \\ \qquad $\eps>0$, inner solver tolerance \\ \qquad $\seq[r\in\N]{\tau_r}$, step size sequence for outer optimization \\ \qquad $\xz$, initialization for inner algorithm
    \\ \qquad $\u\iter{0}$, initialization for outer optimization
    \STATE {\bfseries Output:} \\ \qquad $\u\iter{R}$, estimated outer solution
    \STATE Initialize $\xz (\u\iter{0}) \leftarrow \xz$
    \FOR {$r=0$ {\bfseries to} $R-1$}
    \STATE Initialize $k \leftarrow 0$ \hspace{13em} \textbf{$\triangleright$ Solving the inner problem}
    \REPEAT
    \STATE $\xkp (\u\iter{r}) \leftarrow \fixMap (\xk (\u\iter{r}), \u\iter{r})$
    \STATE $k \leftarrow k + 1$
    \UNTIL $\norm{\xk (\u\iter{r}) - \xkm (\u\iter{r})} \leq \eps$
    \STATE $K_r \leftarrow k$
    \STATE Initialize $\bar\x\iter{K_r} \leftarrow \grad \ell (\x\iter{K_r} (\u\iter{r}))^T$ \hspace{4em} \textbf{$\triangleright$ Computing the hypergradient}
    \STATE Initialize $\bar\u\iter{K_r} \leftarrow 0$
    \FOR {$k=K_r-1$ {\bfseries downto} $0$}
    \STATE $\bar\x\iter{k}_K \leftarrow \bar\x\iter{k}_K \JacVarFixMapk (\u\iter{r})$
    \STATE $\bar\u\iter{k}_K \leftarrow \bar\u\iter{k+1}_K + \bar\x\iter{k}_K \JacPrmFixMapk (\u\iter{r})$
    \ENDFOR
    \STATE Compute $\d\iter{r} \leftarrow (\bar\u\iter{0})^T$

    \STATE $\u\iter{r+1} \leftarrow \u\iter{r} - \tau_r \d\iter{r}$

    \STATE Initialize $\xz (\u\iter{r+1}) \leftarrow \x\iter{K_{r}} (\u\iter{r})$ \hspace{4em} \textbf{$\triangleright$ Warm-starting}
    \ENDFOR
  \end{algorithmic}
\end{Algorithm}

\section{Additional Experiments} \label{sec:experiments:app}
In this section, we provide additional details on the experimental setup and
present supplementary results for varying problem dimensions and truncation
indices.

For any $A \in \R^{M\times N}$ and $\b\in\R^{M}$, we consider the quadratic least-squares problem
\begin{equation}
  f (\x, \u) \coloneqq \frac{1}{2} \norm{A \x - \b}^2 \,,
\end{equation}
where $u=(A,b)$, solved using gradient descent. Let $L$ and $m$ respectively denote the largest and the smallest eigenvalues of $A^T A$. When $m > 0$, the problem is $m$-strongly convex and admits a unique solution expressed analytically by
\begin{equation*}
  \argmap (\u) = (A^T A)^{-1} A^T \b \,,
\end{equation*}
allowing the exact Jacobian of the solution map to be computed for evaluation purposes. For computational efficiency, we generate the Jacobian-vector products $\fwdJack[T] \v$ and vector-Jacobian products $\w^T\revJacPrmk[K+\omega T]$ for $\v$ and $\w$ in $\R^{N}$.

Each element of $A$ is drawn from $U(0, 1)$ while that of $\b$ is sampled from $\mathcal N (0, 1)$. We conduct our experiments for a fix $M = 50$ and multiple problem dimensions $N \in \set{2, 5, 10, 20, 30, 40}$. This affects the conditioning of the problem and hence the convergence rate of gradient descent, allowing us to examine the effect of the convergence rate on the curse of unrolling. For each dimension, we also study the influence of the step size on unrolling by considering two step sizes: the optimal step size $\alpha = 2 / (L + m)$ (Figures~\ref{fig:latestart-dim2-opt}--\ref{fig:latestart-dim40-opt}) and a suboptimal one $\alpha = 1 / (3L)$ (Figures~\ref{fig:latestart-dim2-subopt}--\ref{fig:latestart-dim40-subopt}).

To investigate truncation, we vary the number of initial algorithm iterations excluded from differentiation. We use the budget reallocation strategy described in Section~\ref{ssec:truncation:realloc} and choose $\omega = 3.0$. In particular, for a fixed dimension, all methods are compared under an equal computational budget, with reductions in derivative computations reallocated to additional forward iterations according to $\omega$. This ensures that observed improvements are attributable to truncation rather than increased computation.

For each problem dimension $N$, we use $9$ different values of $T$ from $\set{0, 0.1, 0.2, \ldots, 0.8}K$. For $T = 0$, we run our algorithm for $K \coloneqq \min(\ceil{\log(10^{-3}) / \log(\rho^*)}, 1000)$ iterations where $\rho^* \coloneqq (L - m) / (L + m)$. To keep the computational budget fixed, the algorithm is run for $K + \omega T$ for a given $T$. We start with the same initialization for $\fwdJacz[T]\v = 0$ for all $T$.

We use the linear solver of PyTorch \citep{PGM+19} for computing $\argmap (\u)$ and the corresponding Jacobians. We perform each experiment $100$ times to generate $100$ sequences $\itrErrk$, $\fwdErrk[T]$ and $\revErrk[T]$ and plot the median accros the $100$ experiments.

Each figure corresponds to a single step size $\alpha$ and problem dimension $N$. The errors for forward mode are shown in the left subfigure while those for reverse mode error are depicted in the right subfigure. For a better visualization of the effect of truncation, we plot all the late-started or truncated derivative error plots in the same figure for a given $N$ and $\alpha$.

In each subfigure, the dashed blue curves show the error of the algorithmic iterates for $k = 0, \ldots, K + \omega T_{\max}$. For each truncation index $T$, the solid curves represent the error of the corresponding derivative sequences, plotted over the range $k = T + \omega T, \ldots, K + \omega T$. The dashed black curves with circular markers indicate the final error attained by each derivative algorithm. Specifically, final errors are reported at $k = K + \omega T$ for
forward mode and at $k = T + \omega T$ for reverse mode. The circular marker on the algorithm error curve denotes the final algorithmic error corresponding to the chosen truncation index $T$. In addition, the dashed curves with cross-shaped markers in the left-hand figures depict the upper bound given in \eqref{prob:optimal:T}.

For the optimal step size and large convergence rates $\rho$ (corresponding to $N = 30$ and $N = 40$), the curse of unrolling is more pronounced, as the algorithm requires substantially more iterations to converge. In this regime, the derivative error in both forward and reverse modes increases as additional derivative recursions are performed, making $T_{\max}$ the preferred choice. As the convergence rate improves, that is, as $\rho$ decreases, the severity of the curse diminishes and the optimal truncation index shifts away from $T_{\max}$. Moreover, increasing the truncation index progressively suppresses the curse, as illustrated in Figures~\ref{fig:latestart-dim2-opt}--\ref{fig:latestart-dim20-opt}.

Finally, across all dimensions, using a suboptimal step size slows the convergence but also reduces the severity of the curse of unrolling, which confirms the findings of Scieur et al. \citep{SGB+22}. Moreover, truncation is ineffective in this regime, since even without truncation the derivative error does not increase substantially during the early iterations.

These additional experiments reinforce the conclusions of the main text by
demonstrating that the curse of unrolling is a non-asymptotic phenomenon governed by convergence rate, smoothness, and initialization, and that truncation provides an effective and robust mechanism for mitigating it across a range of problem scales.

\clearpage
\begin{figure}[t]
  \readdef{figures/dim=2/rates.txt}{\mydata}
  \readarray\mydata\ratearray[-,\ncols]
  \centering
  \begin{subfigure}[t]{0.49\textwidth}
    \centering
    \includegraphics[width=\linewidth]{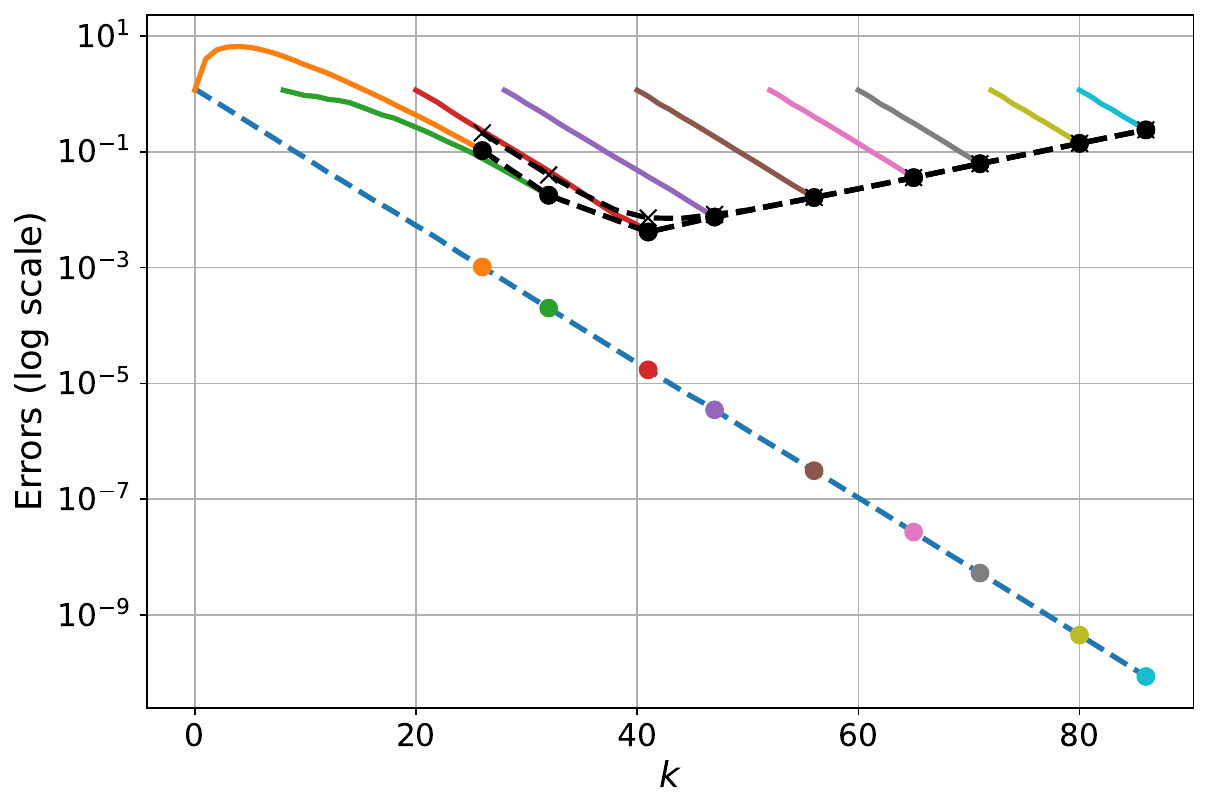}
    \caption{Forward mode.}
    \label{subfig:latestart-dim2-opt-fw}
  \end{subfigure}\hfill
  \begin{subfigure}[t]{0.49\textwidth}
    \centering
    \includegraphics[width=\linewidth]{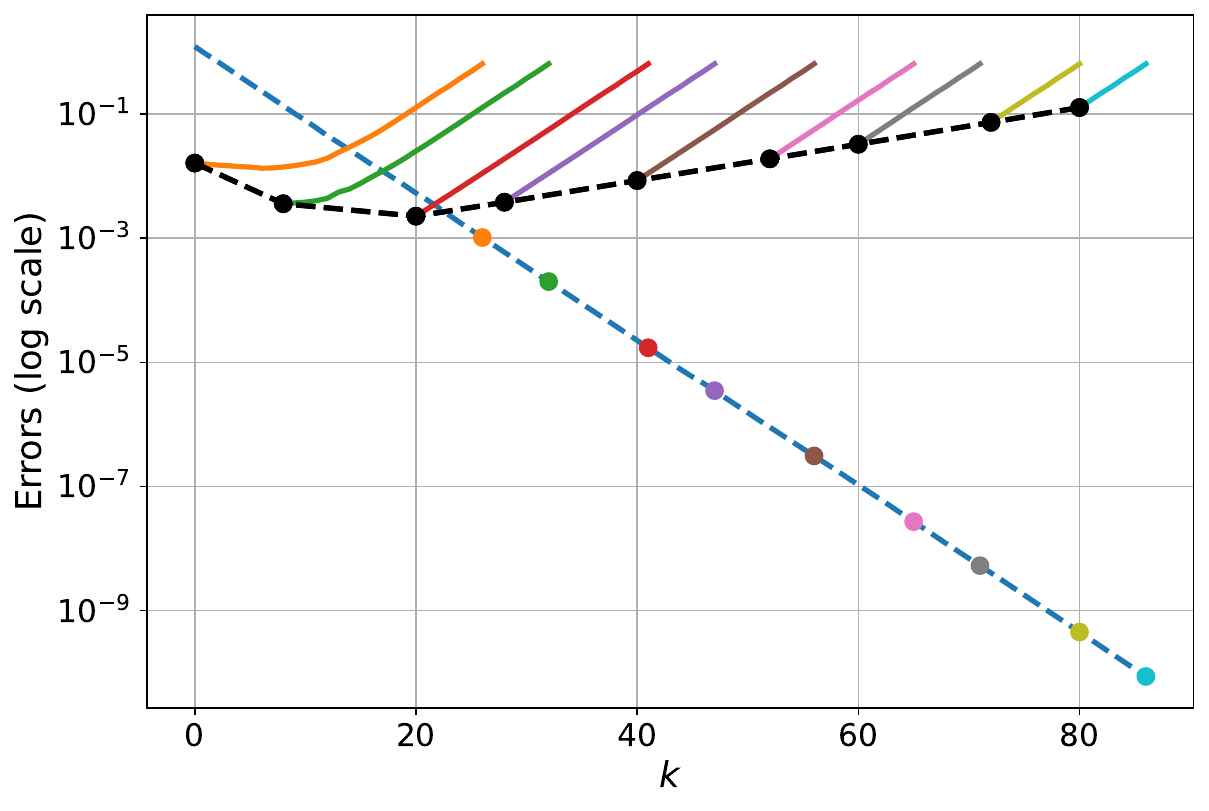}
    \caption{Reverse mode.}
    \label{subfig:latestart-dim2-opt-bw}
  \end{subfigure}
  \caption{Late-start / truncation behavior for $N = 2$, $\alpha = 2 / (L + m)$, and $\rho \approx \fpeval{round(\ratearray[1,2], 6)}$.}
  \label{fig:latestart-dim2-opt}
\end{figure}

\begin{figure}[t]
  \readdef{figures/dim=5/rates.txt}{\mydata}
  \readarray\mydata\ratearray[-,\ncols]
  \centering
  \begin{subfigure}[t]{0.49\textwidth}
    \centering
    \includegraphics[width=\linewidth]{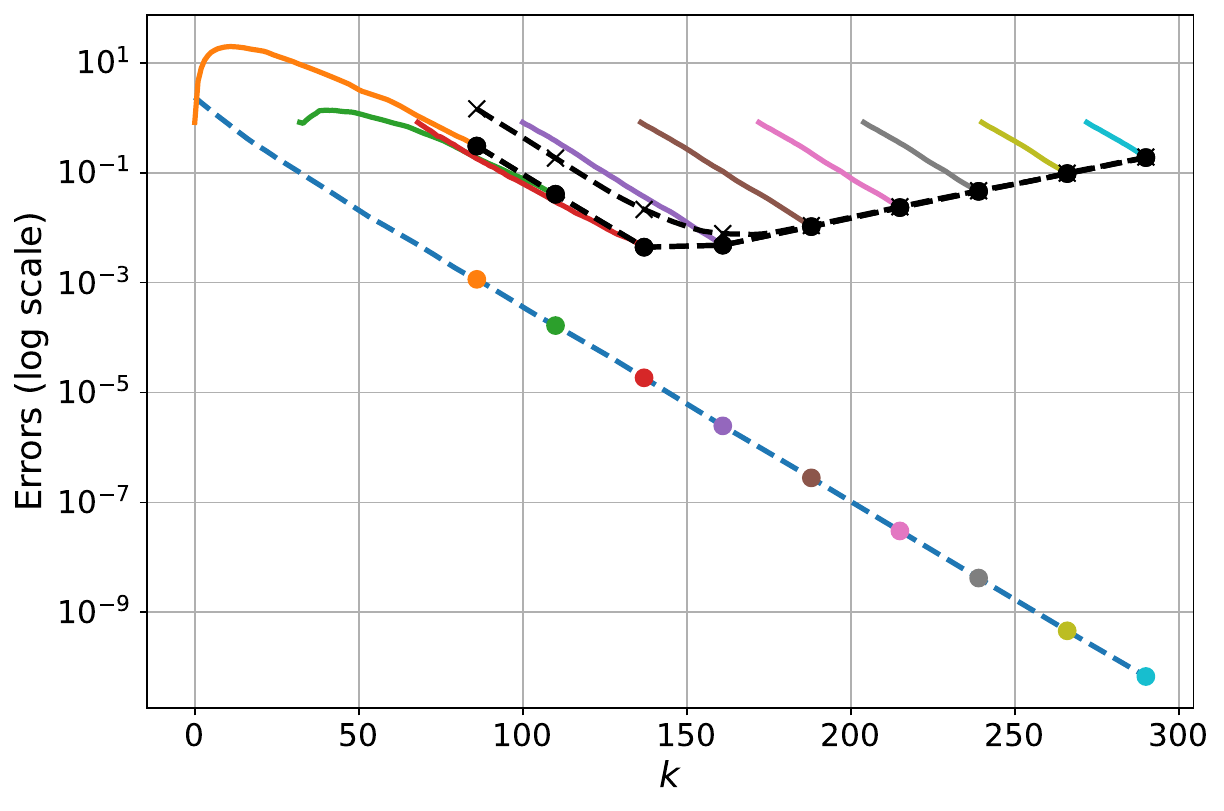}
    \caption{Forward mode.}
    \label{subfig:latestart-dim5-opt-fw}
  \end{subfigure}\hfill
  \begin{subfigure}[t]{0.49\textwidth}
    \centering
    \includegraphics[width=\linewidth]{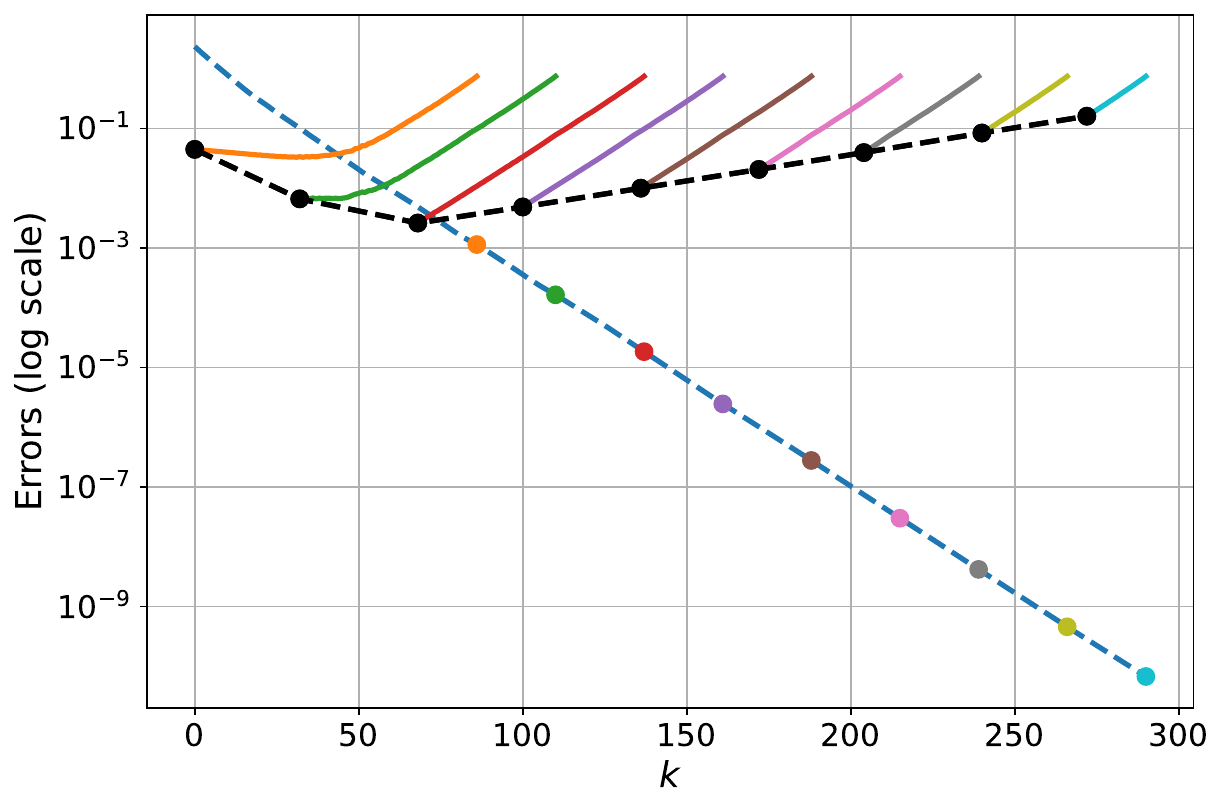}
    \caption{Reverse mode.}
    \label{subfig:latestart-dim5-opt-bw}
  \end{subfigure}
  \caption{Late-start / truncation behavior for $N = 5$, $\alpha = 2 / (L + m)$, and $\rho \approx \fpeval{round(\ratearray[1,2], 6)}$.}
  \label{fig:latestart-dim5-opt}
\end{figure}

\begin{figure}[t]
  \readdef{figures/dim=10/rates.txt}{\mydata}
  \readarray\mydata\ratearray[-,\ncols]
  \centering
  \begin{subfigure}[t]{0.49\textwidth}
    \centering
    \includegraphics[width=\linewidth]{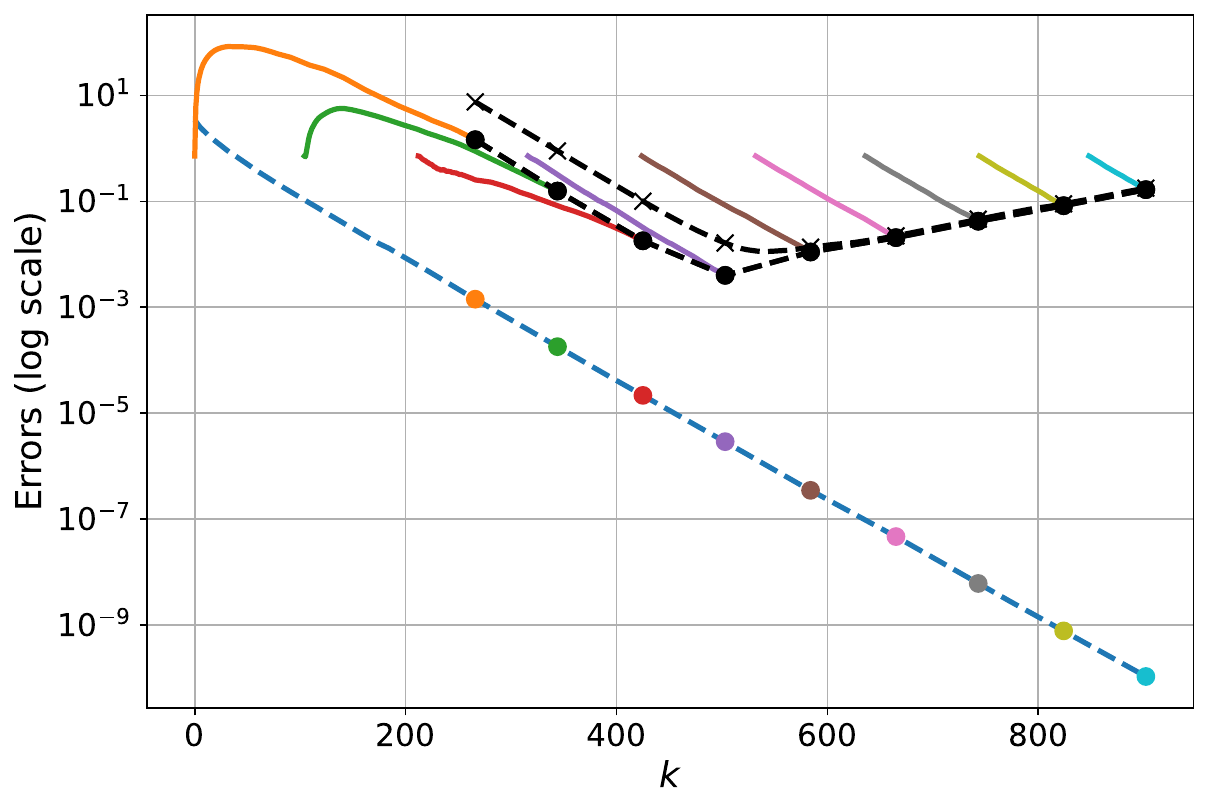}
    \caption{Forward mode.}
    \label{subfig:latestart-dim10-opt-fw}
  \end{subfigure}\hfill
  \begin{subfigure}[t]{0.49\textwidth}
    \centering
    \includegraphics[width=\linewidth]{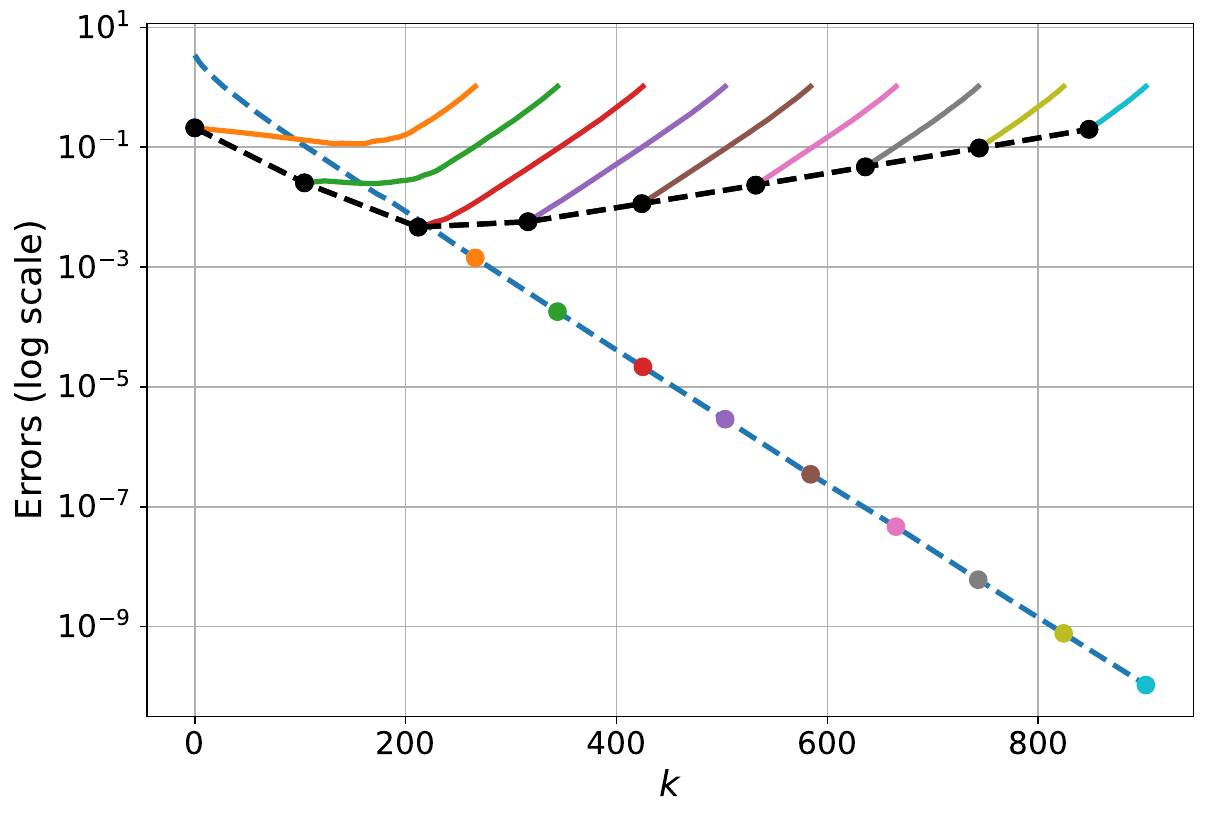}
    \caption{Reverse mode.}
    \label{subfig:latestart-dim10-opt-bw}
  \end{subfigure}
  \caption{Late-start / truncation behavior for $N = 10$, $\alpha = 2 / (L + m)$, and $\rho \approx \fpeval{round(\ratearray[1,2], 6)}$.}
  \label{fig:latestart-dim10-opt}
\end{figure}

\begin{figure}[t]
  \readdef{figures/dim=20/rates.txt}{\mydata}
  \readarray\mydata\ratearray[-,\ncols]
  \centering
  \begin{subfigure}[t]{0.49\textwidth}
    \centering
    \includegraphics[width=\linewidth]{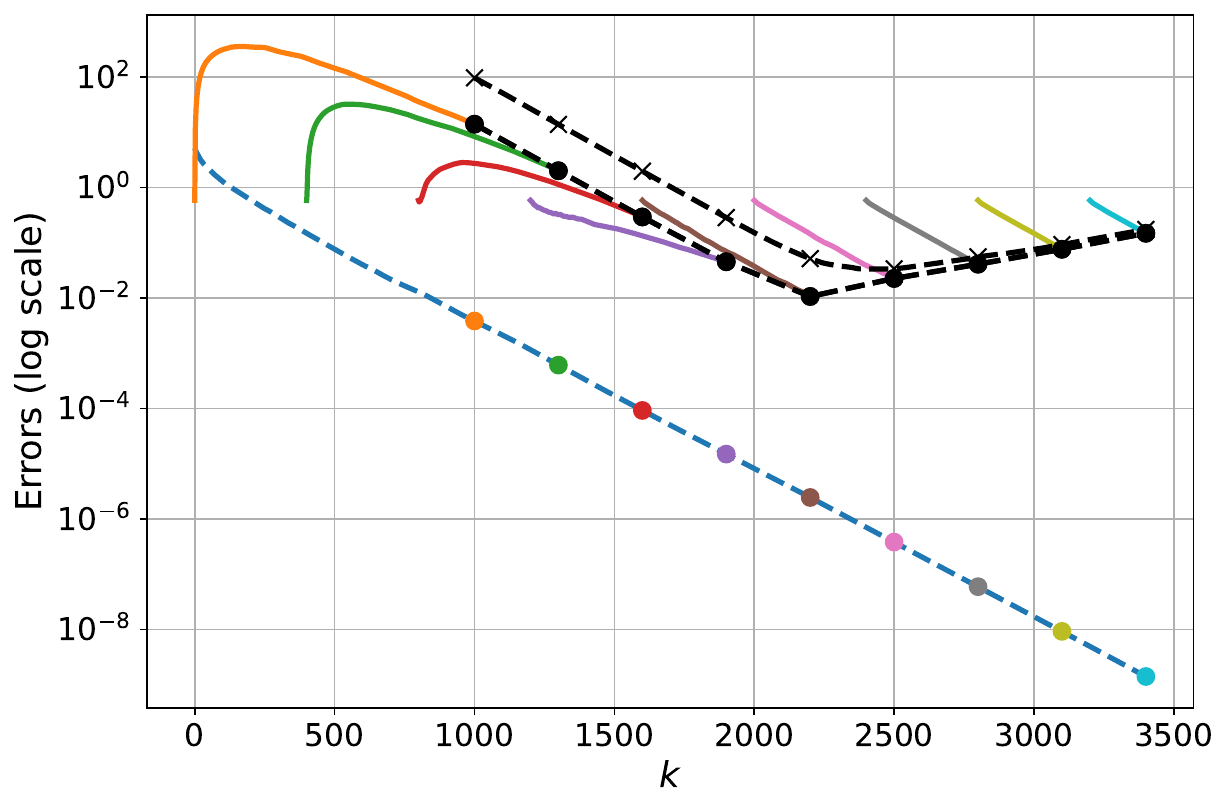}
    \caption{Forward mode.}
    \label{subfig:latestart-dim20-opt-fw}
  \end{subfigure}\hfill
  \begin{subfigure}[t]{0.49\textwidth}
    \centering
    \includegraphics[width=\linewidth]{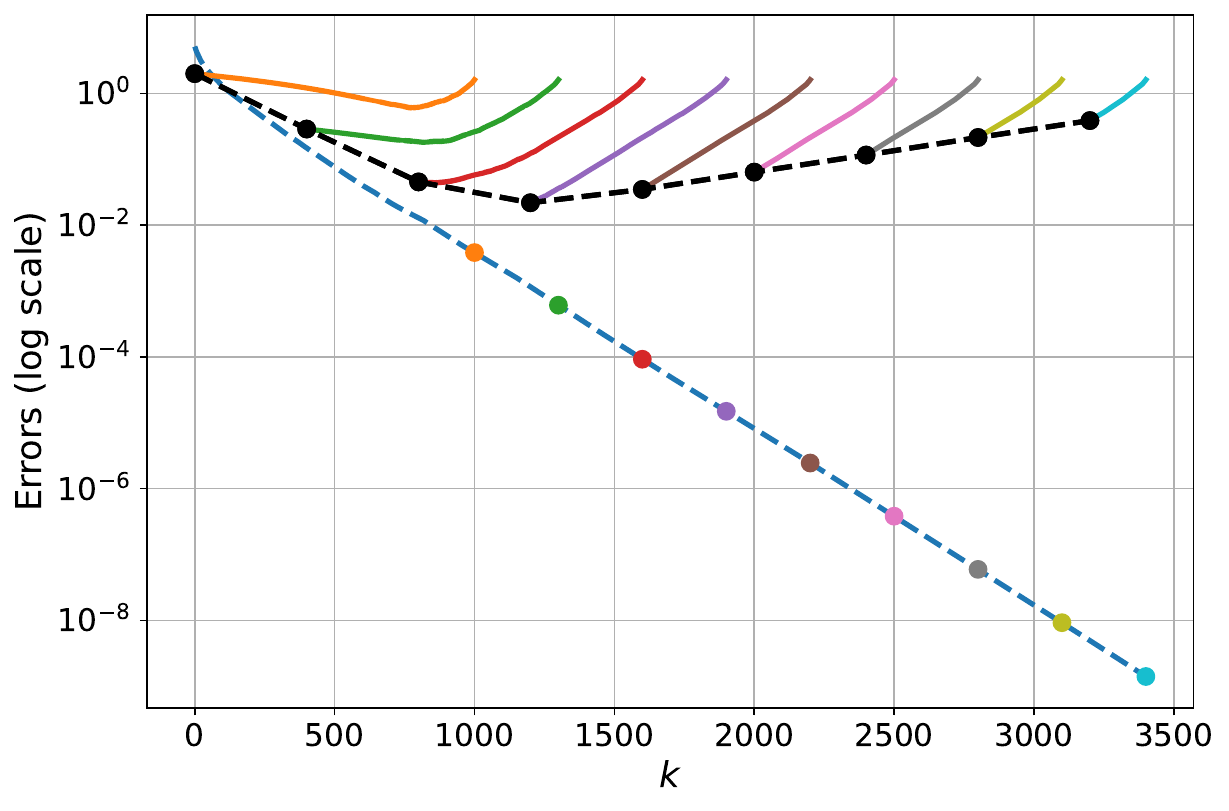}
    \caption{Reverse mode.}
    \label{subfig:latestart-dim20-opt-bw}
  \end{subfigure}
  \caption{Late-start / truncation behavior for $N = 20$, $\alpha = 2 / (L + m)$, and $\rho \approx \fpeval{round(\ratearray[1,2], 6)}$.}
  \label{fig:latestart-dim20-opt}
\end{figure}

\begin{figure}[t]
  \readdef{figures/dim=30/rates.txt}{\mydata}
  \readarray\mydata\ratearray[-,\ncols]
  \centering
  \begin{subfigure}[t]{0.49\textwidth}
    \centering
    \includegraphics[width=\linewidth]{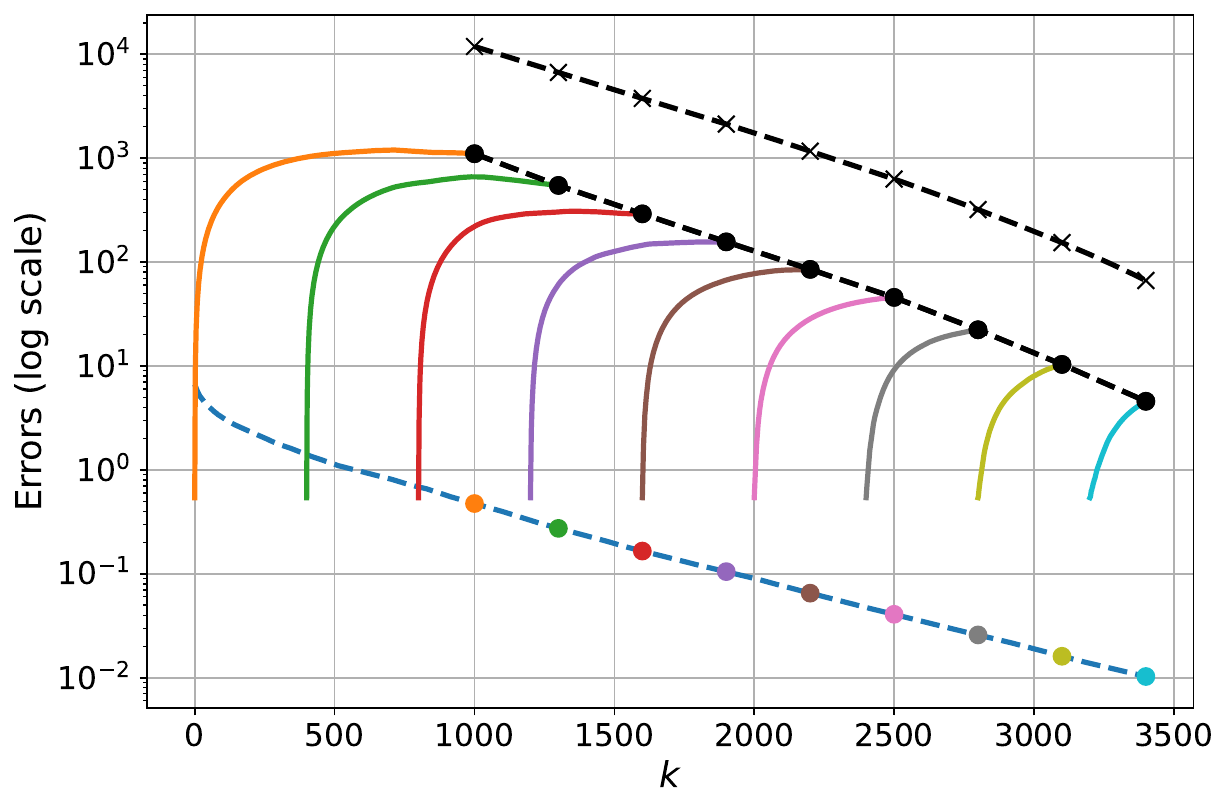}
    \caption{Forward mode.}
    \label{subfig:latestart-dim30-opt-fw}
  \end{subfigure}\hfill
  \begin{subfigure}[t]{0.49\textwidth}
    \centering
    \includegraphics[width=\linewidth]{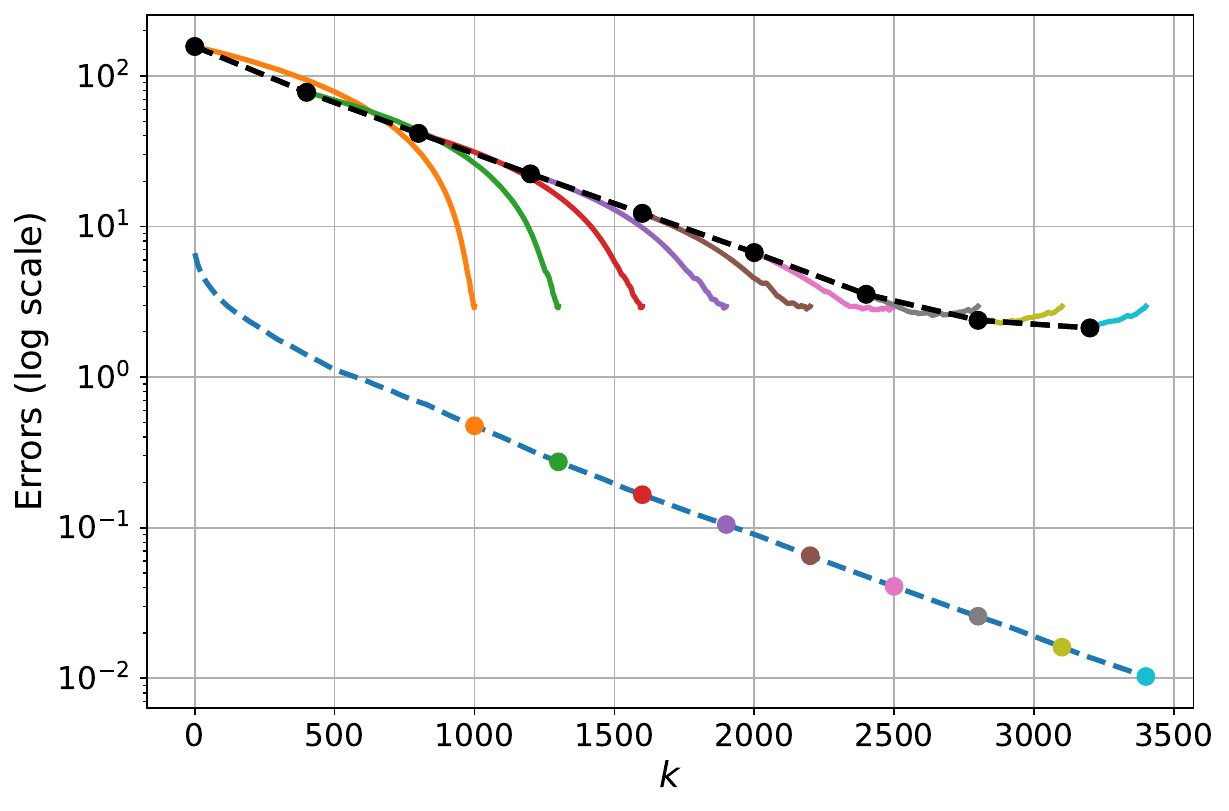}
    \caption{Reverse mode.}
    \label{subfig:latestart-dim30-opt-bw}
  \end{subfigure}
  \caption{Late-start / truncation behavior for $N = 30$, $\alpha = 2 / (L + m)$, and $\rho \approx \fpeval{round(\ratearray[1,2], 6)}$.}
  \label{fig:latestart-dim30-opt}
\end{figure}

\begin{figure}[t]
  \readdef{figures/dim=40/rates.txt}{\mydata}
  \readarray\mydata\ratearray[-,\ncols]
  \centering
  \begin{subfigure}[t]{0.49\textwidth}
    \centering
    \includegraphics[width=\linewidth]{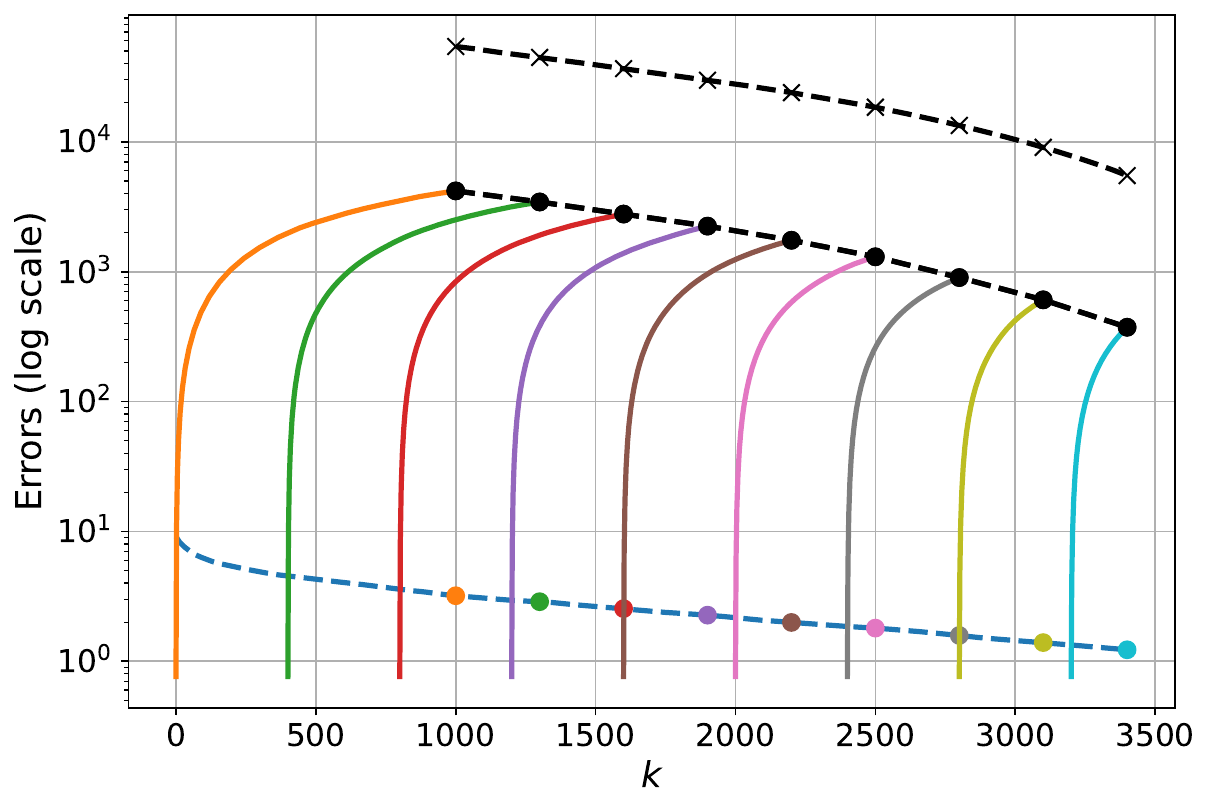}
    \caption{Forward mode.}
    \label{subfig:latestart-dim40-opt-fw}
  \end{subfigure}\hfill
  \begin{subfigure}[t]{0.49\textwidth}
    \centering
    \includegraphics[width=\linewidth]{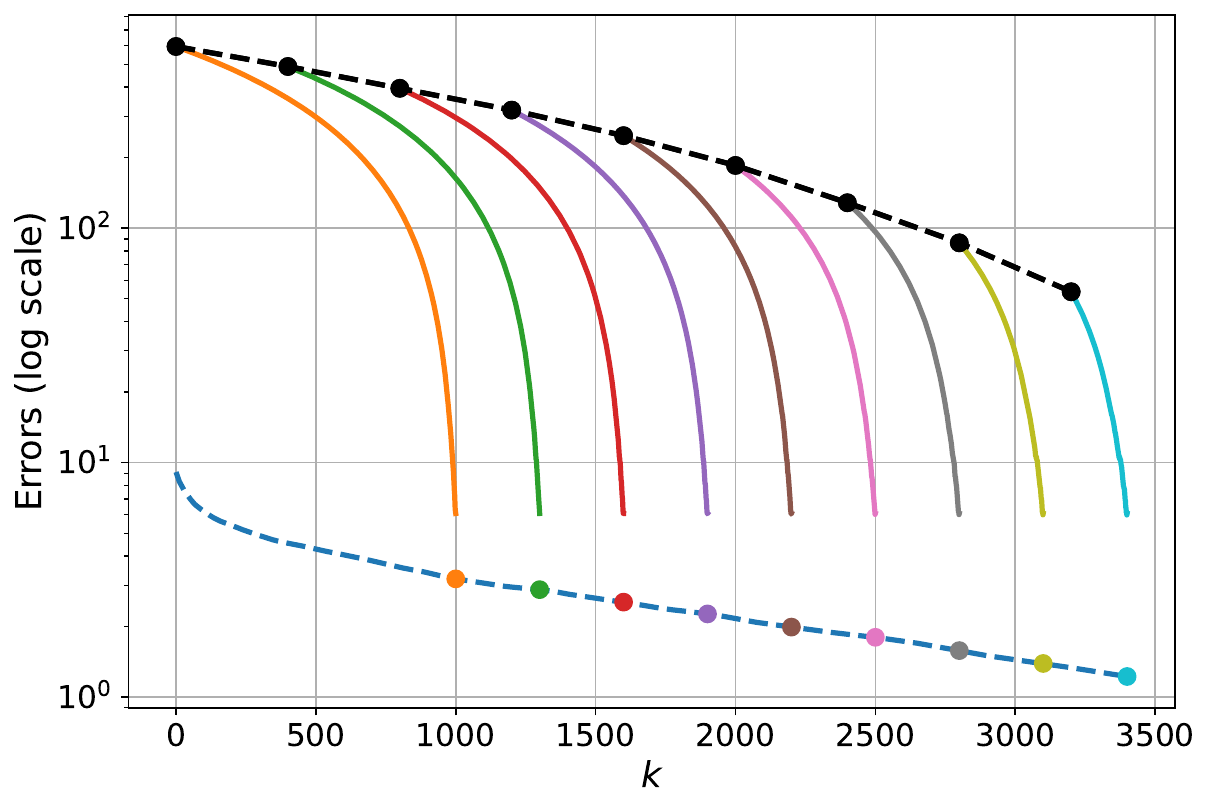}
    \caption{Reverse mode.}
    \label{subfig:latestart-dim40-opt-bw}
  \end{subfigure}
  \caption{Late-start / truncation behavior for $N = 40$, $\alpha = 2 / (L + m)$, and $\rho \approx \fpeval{round(\ratearray[1,2], 6)}$.}
  \label{fig:latestart-dim40-opt}
\end{figure}

\begin{figure}[t]
  \readdef{figures/dim=2/rates.txt}{\mydata}
  \readarray\mydata\ratearray[-,\ncols]
  \centering
  \begin{subfigure}[t]{0.49\textwidth}
    \centering
    \includegraphics[width=\linewidth]{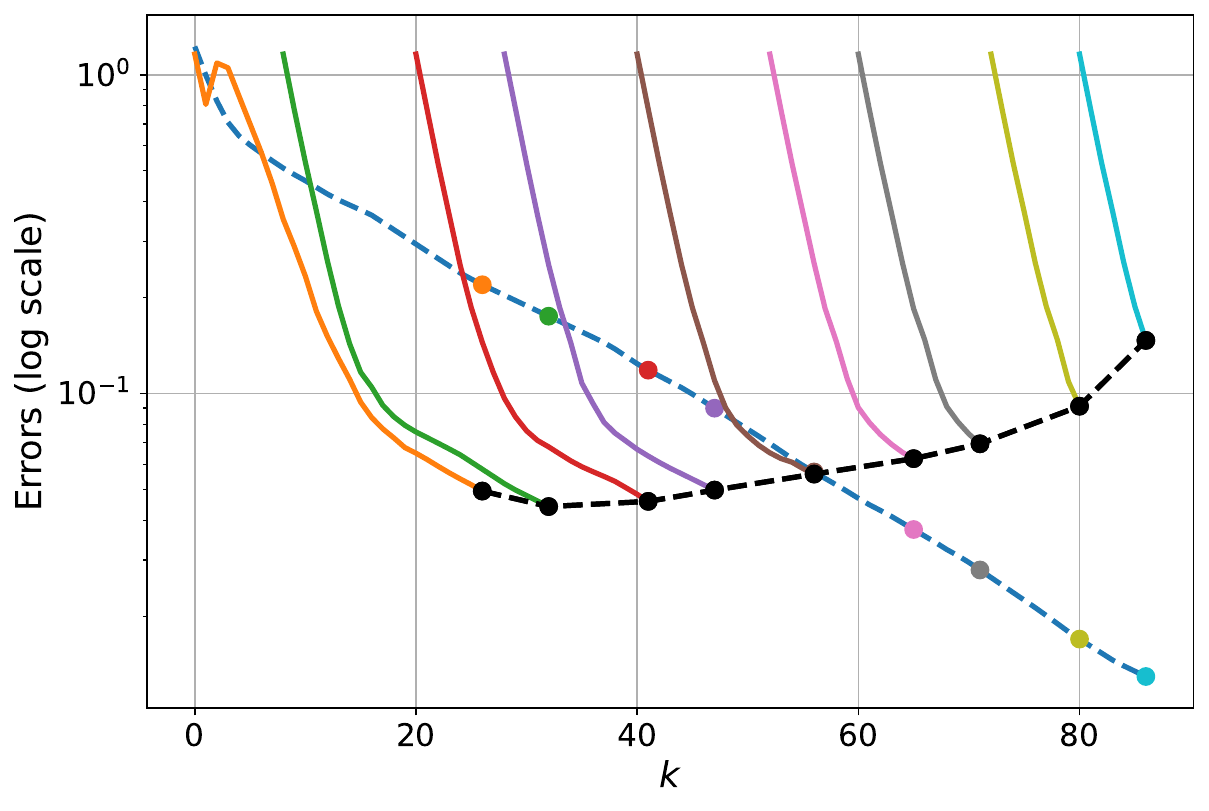}
    \caption{Forward mode.}
    \label{subfig:latestart-dim2-subopt-fw}
  \end{subfigure}\hfill
  \begin{subfigure}[t]{0.49\textwidth}
    \centering
    \includegraphics[width=\linewidth]{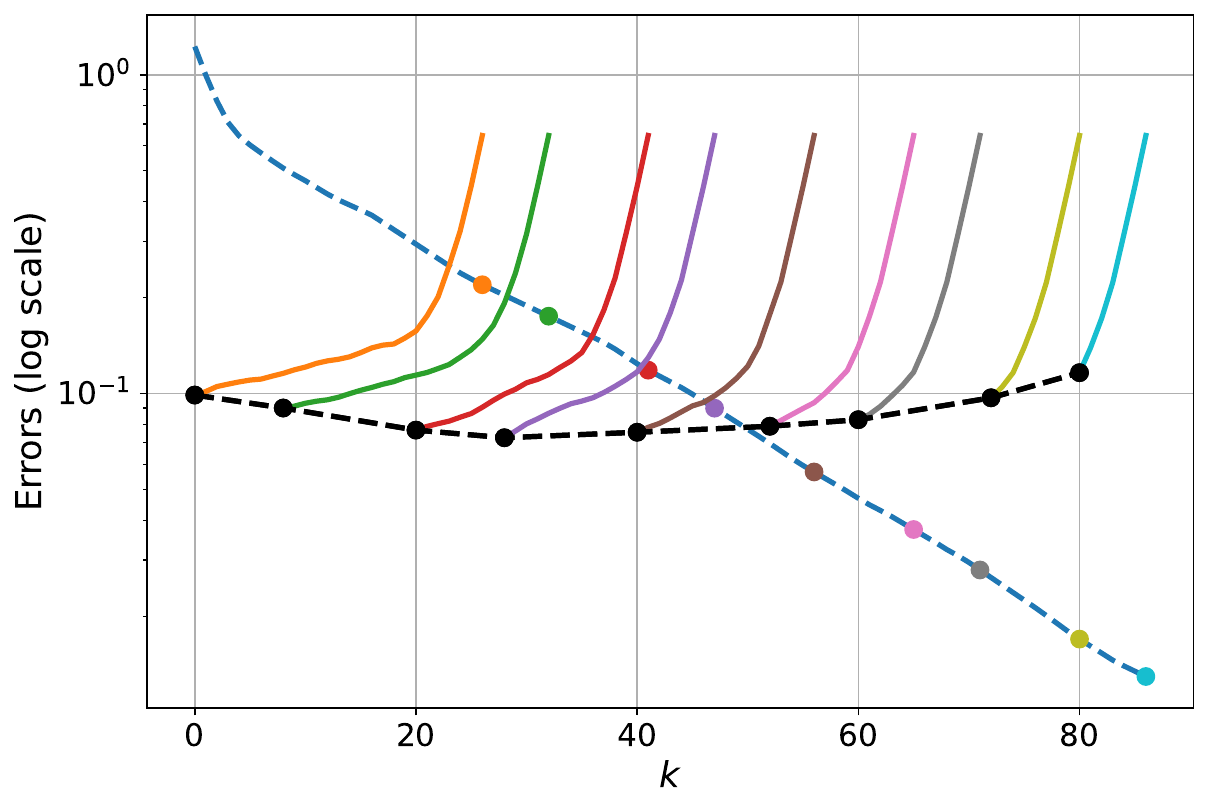}
    \caption{Reverse mode.}
    \label{subfig:latestart-dim2-subopt-bw}
  \end{subfigure}
  \caption{Late-start / truncation behavior for $N = 2$, $\alpha = 1 / (3L)$, and $\rho \approx \fpeval{round(\ratearray[1,2], 6)}$.}
  \label{fig:latestart-dim2-subopt}
\end{figure}

\begin{figure}[t]
  \readdef{figures/dim=5/rates.txt}{\mydata}
  \readarray\mydata\ratearray[-,\ncols]
  \centering
  \begin{subfigure}[t]{0.49\textwidth}
    \centering
    \includegraphics[width=\linewidth]{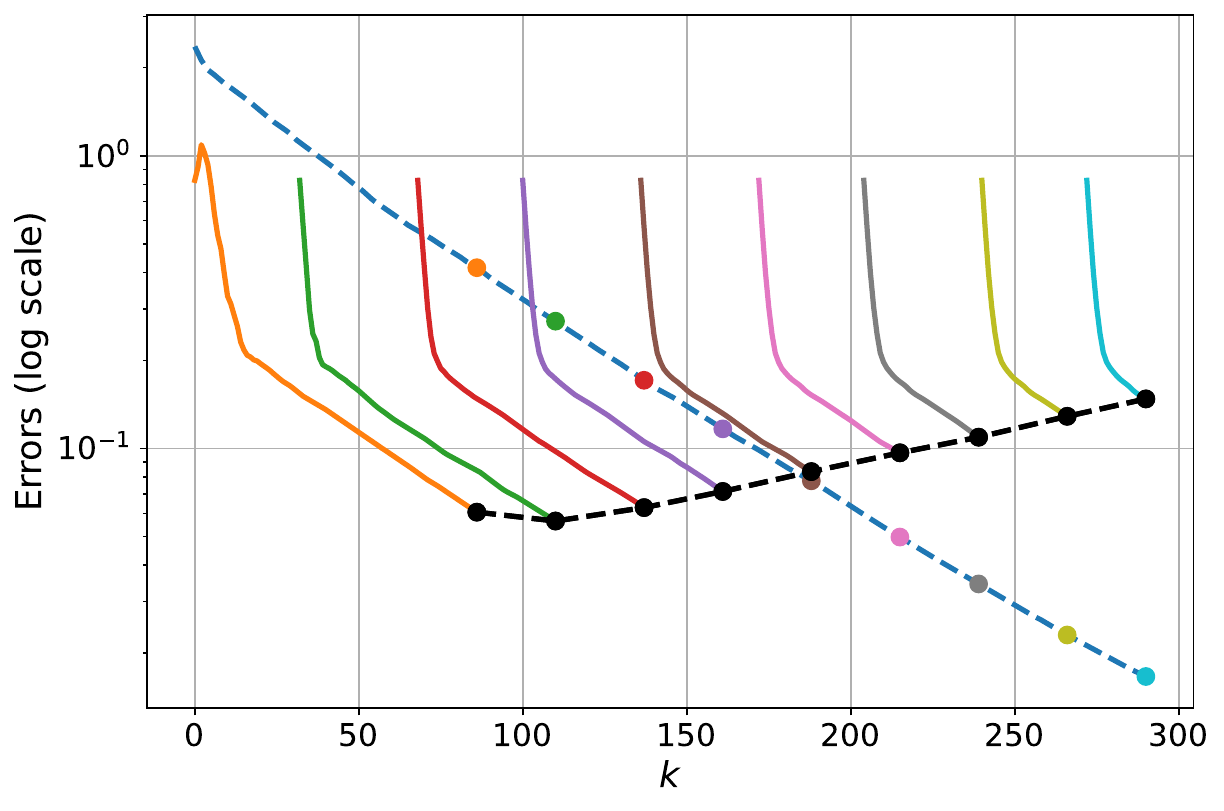}
    \caption{Forward mode.}
    \label{subfig:latestart-dim5-subopt-fw}
  \end{subfigure}\hfill
  \begin{subfigure}[t]{0.49\textwidth}
    \centering
    \includegraphics[width=\linewidth]{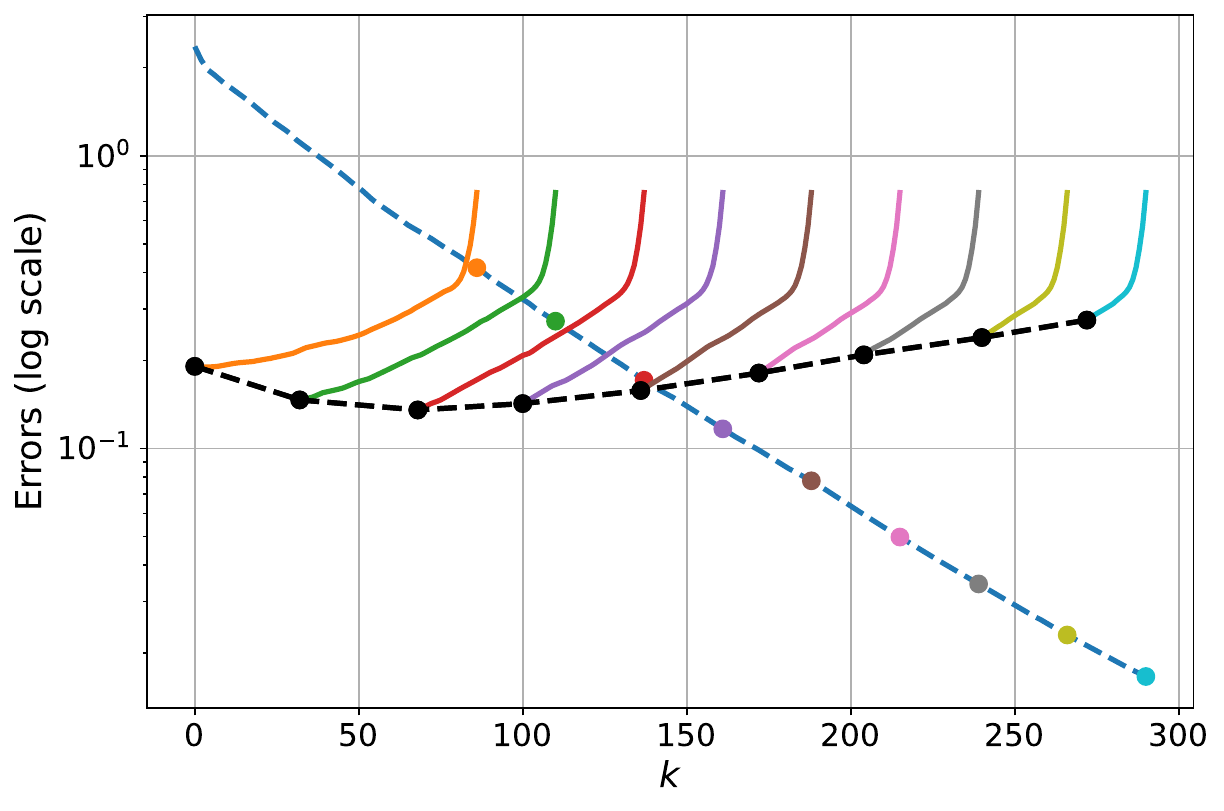}
    \caption{Reverse mode.}
    \label{subfig:latestart-dim5-subopt-bw}
  \end{subfigure}
  \caption{Late-start / truncation behavior for $N = 5$, $\alpha = 1 / (3L)$, and $\rho \approx \fpeval{round(\ratearray[1,2], 6)}$.}
  \label{fig:latestart-dim5-subopt}
\end{figure}

\begin{figure}[t]
  \readdef{figures/dim=10/rates.txt}{\mydata}
  \readarray\mydata\ratearray[-,\ncols]
  \centering
  \begin{subfigure}[t]{0.49\textwidth}
    \centering
    \includegraphics[width=\linewidth]{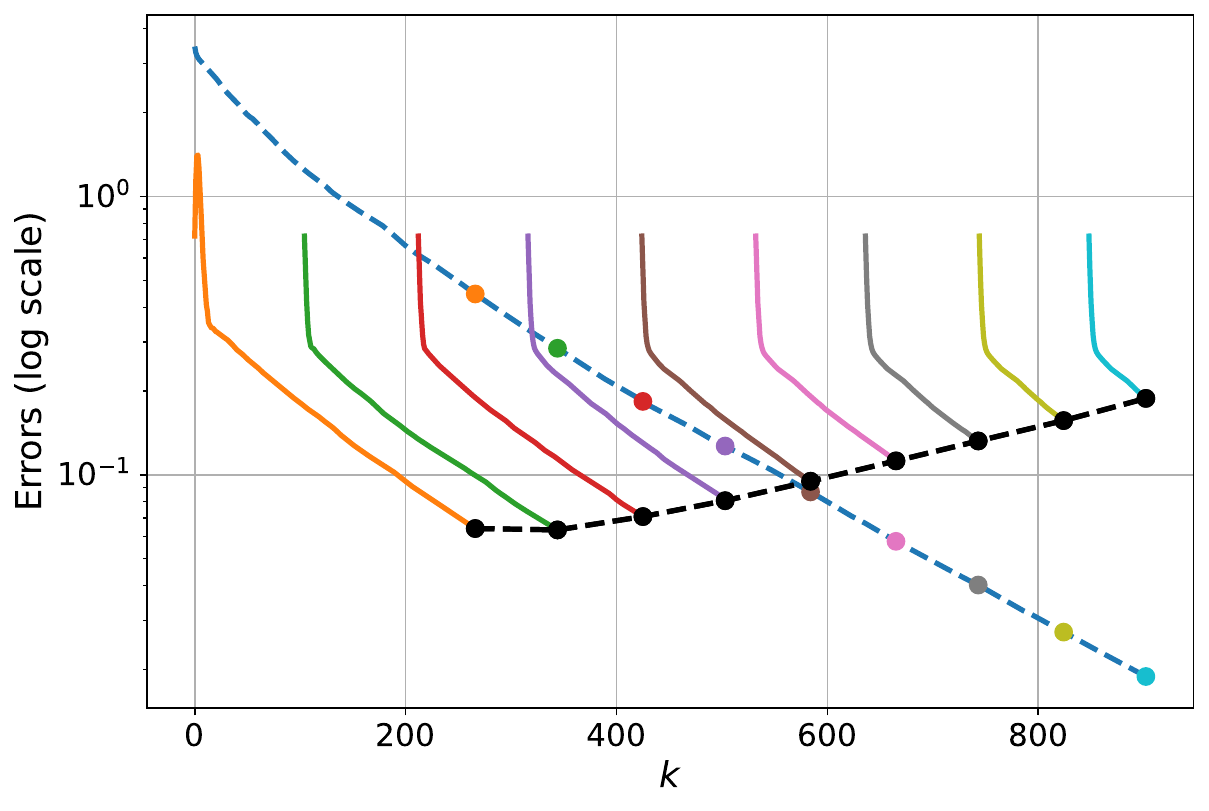}
    \caption{Forward mode.}
    \label{subfig:latestart-dim10-subopt-fw}
  \end{subfigure}\hfill
  \begin{subfigure}[t]{0.49\textwidth}
    \centering
    \includegraphics[width=\linewidth]{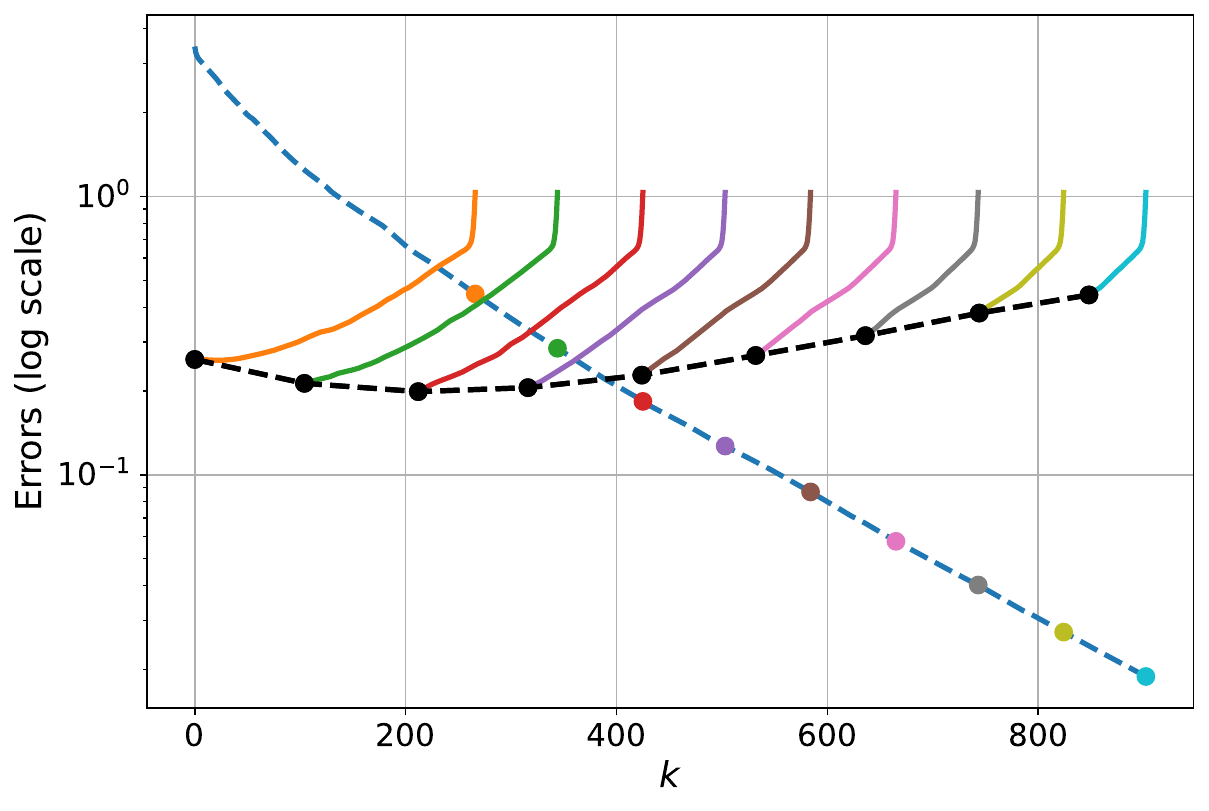}
    \caption{Reverse mode.}
    \label{subfig:latestart-dim10-subopt-bw}
  \end{subfigure}
  \caption{Late-start / truncation behavior for $N = 10$, $\alpha = 1 / (3L)$, and $\rho \approx \fpeval{round(\ratearray[1,2], 6)}$.}
  \label{fig:latestart-dim10-subopt}
\end{figure}

\begin{figure}[t]
  \readdef{figures/dim=20/rates.txt}{\mydata}
  \readarray\mydata\ratearray[-,\ncols]
  \centering
  \begin{subfigure}[t]{0.49\textwidth}
    \centering
    \includegraphics[width=\linewidth]{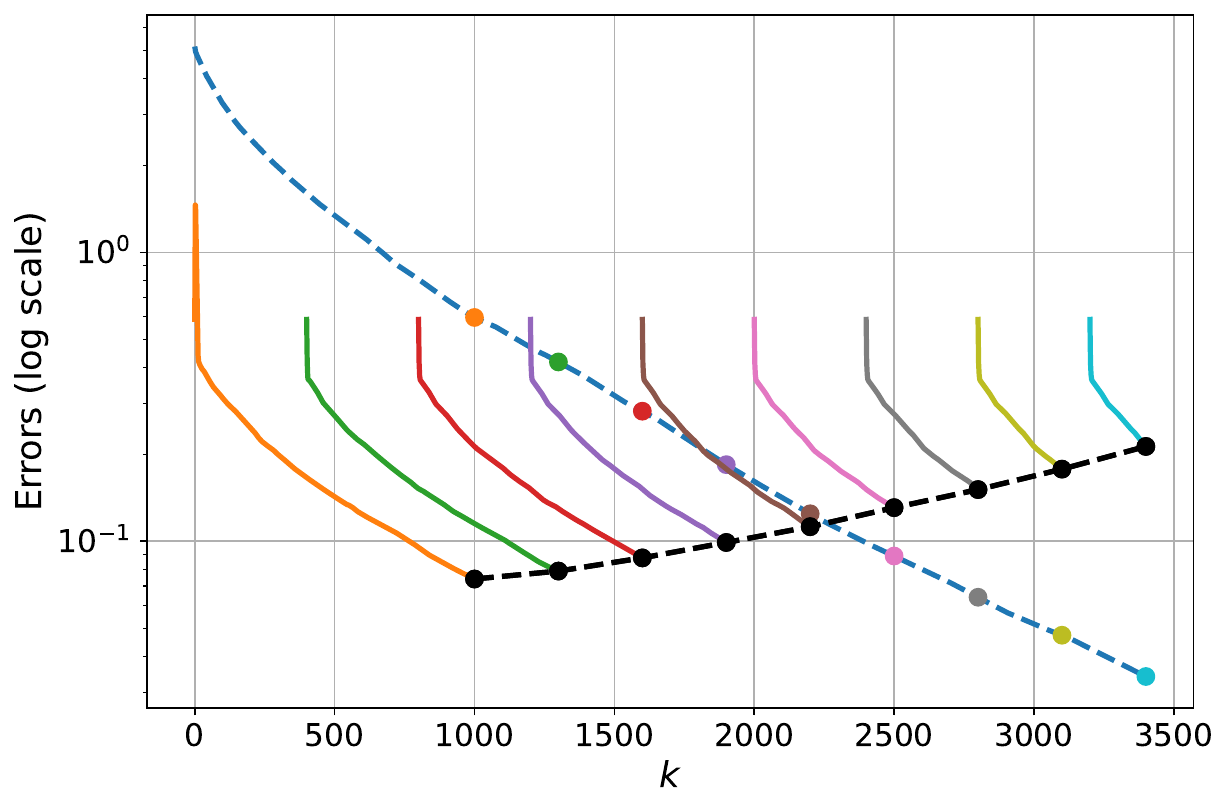}
    \caption{Forward mode.}
    \label{subfig:latestart-dim20-subopt-fw}
  \end{subfigure}\hfill
  \begin{subfigure}[t]{0.49\textwidth}
    \centering
    \includegraphics[width=\linewidth]{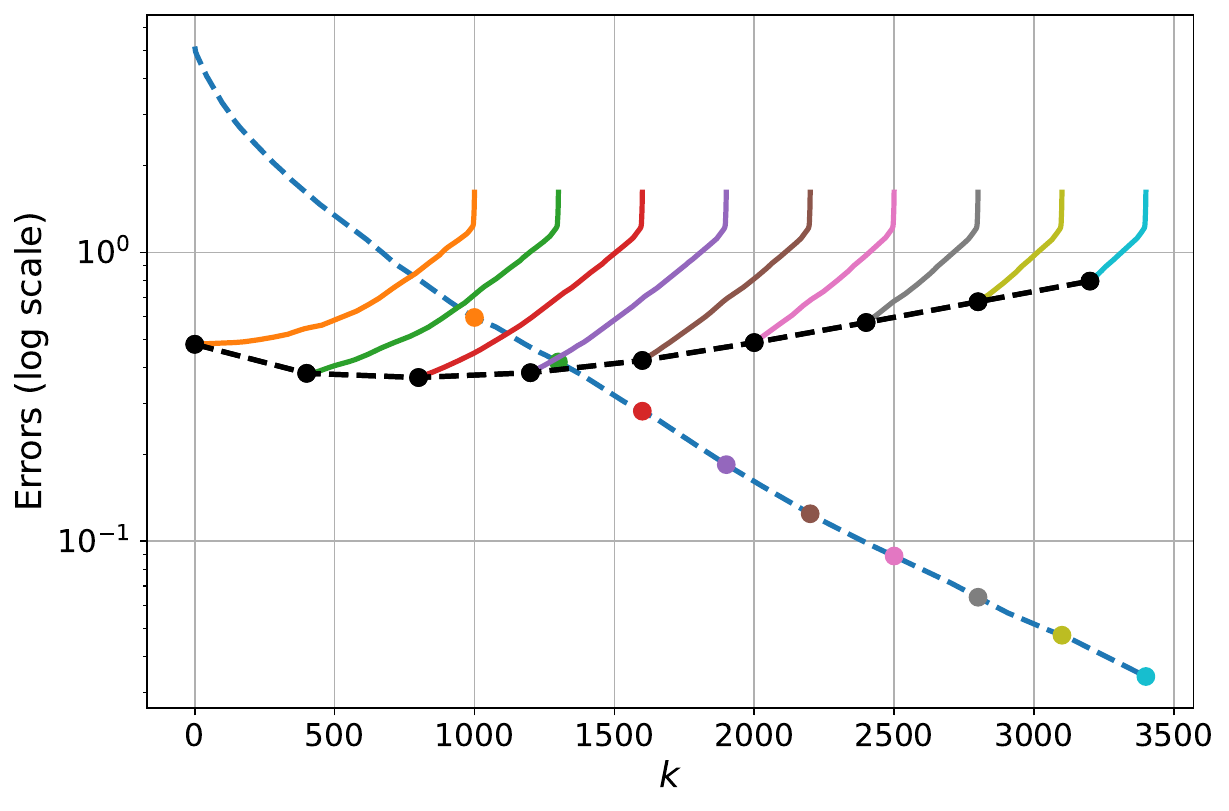}
    \caption{Reverse mode.}
    \label{subfig:latestart-dim20-subopt-bw}
  \end{subfigure}
  \caption{Late-start / truncation behavior for $N = 20$, $\alpha = 1 / (3L)$, and $\rho \approx \fpeval{round(\ratearray[1,2], 6)}$.}
  \label{fig:latestart-dim20-subopt}
\end{figure}

\begin{figure}[t]
  \readdef{figures/dim=30/rates.txt}{\mydata}
  \readarray\mydata\ratearray[-,\ncols]
  \centering
  \begin{subfigure}[t]{0.49\textwidth}
    \centering
    \includegraphics[width=\linewidth]{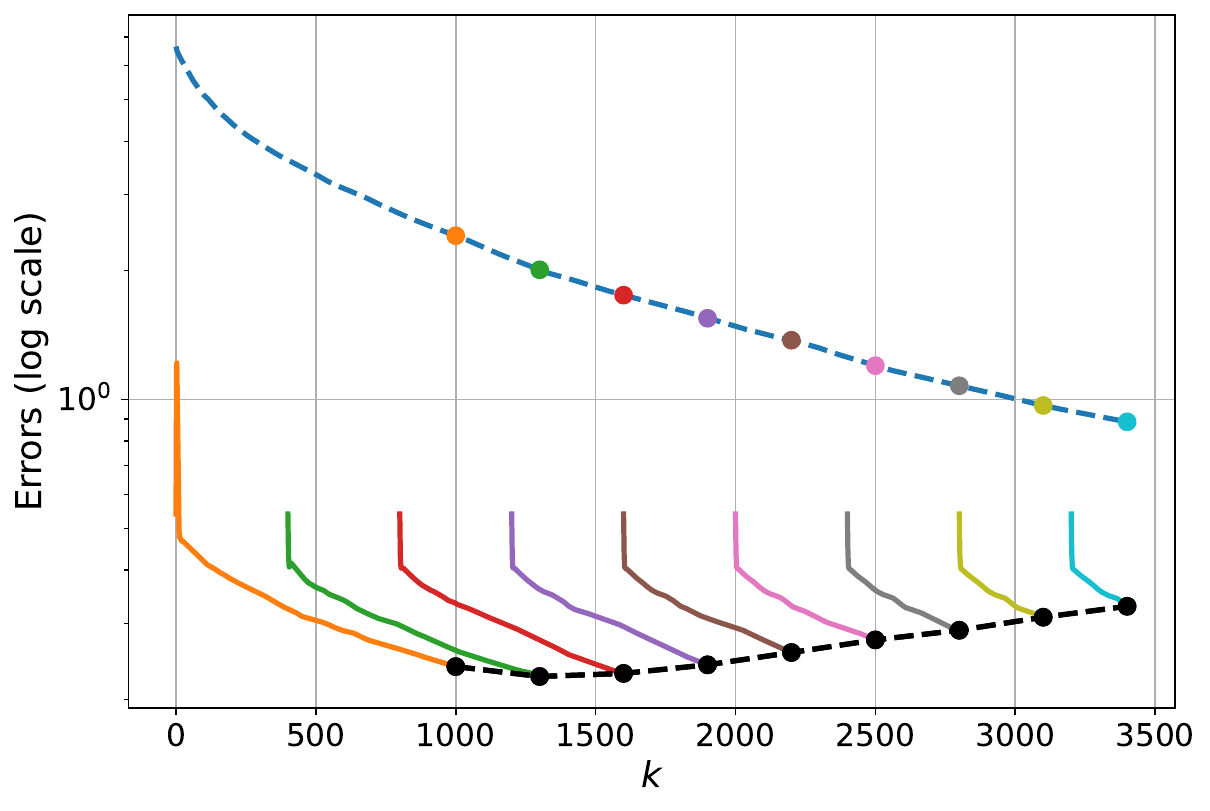}
    \caption{Forward mode.}
    \label{subfig:latestart-dim30-subopt-fw}
  \end{subfigure}\hfill
  \begin{subfigure}[t]{0.49\textwidth}
    \centering
    \includegraphics[width=\linewidth]{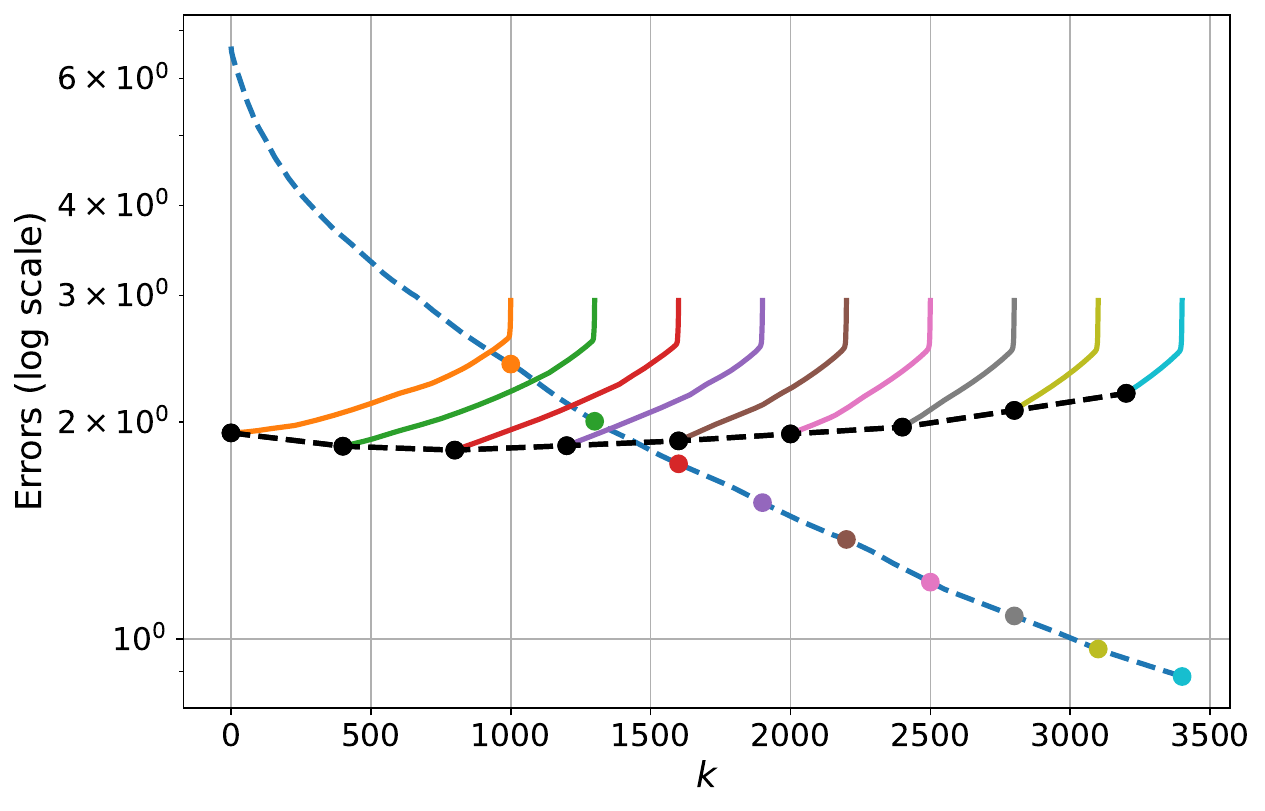}
    \caption{Reverse mode.}
    \label{subfig:latestart-dim30-subopt-bw}
  \end{subfigure}
  \caption{Late-start / truncation behavior for $N = 30$, $\alpha = 1 / (3L)$, and $\rho \approx \fpeval{round(\ratearray[1,2], 6)}$.}
  \label{fig:latestart-dim30-subopt}
\end{figure}

\begin{figure}[t]
  \readdef{figures/dim=40/rates.txt}{\mydata}
  \readarray\mydata\ratearray[-,\ncols]
  \centering
  \begin{subfigure}[t]{0.49\textwidth}
    \centering
    \includegraphics[width=\linewidth]{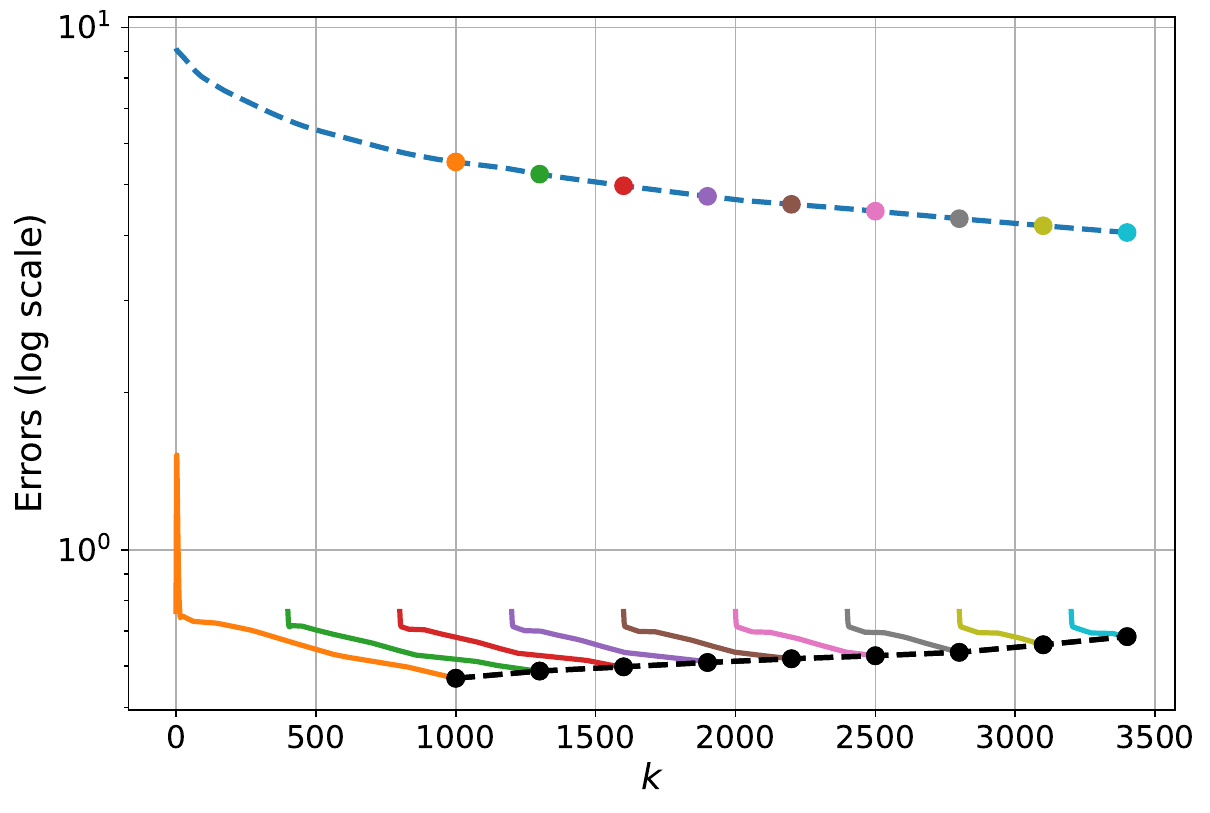}
    \caption{Forward mode.}
    \label{subfig:latestart-dim40-subopt-fw}
  \end{subfigure}\hfill
  \begin{subfigure}[t]{0.49\textwidth}
    \centering
    \includegraphics[width=\linewidth]{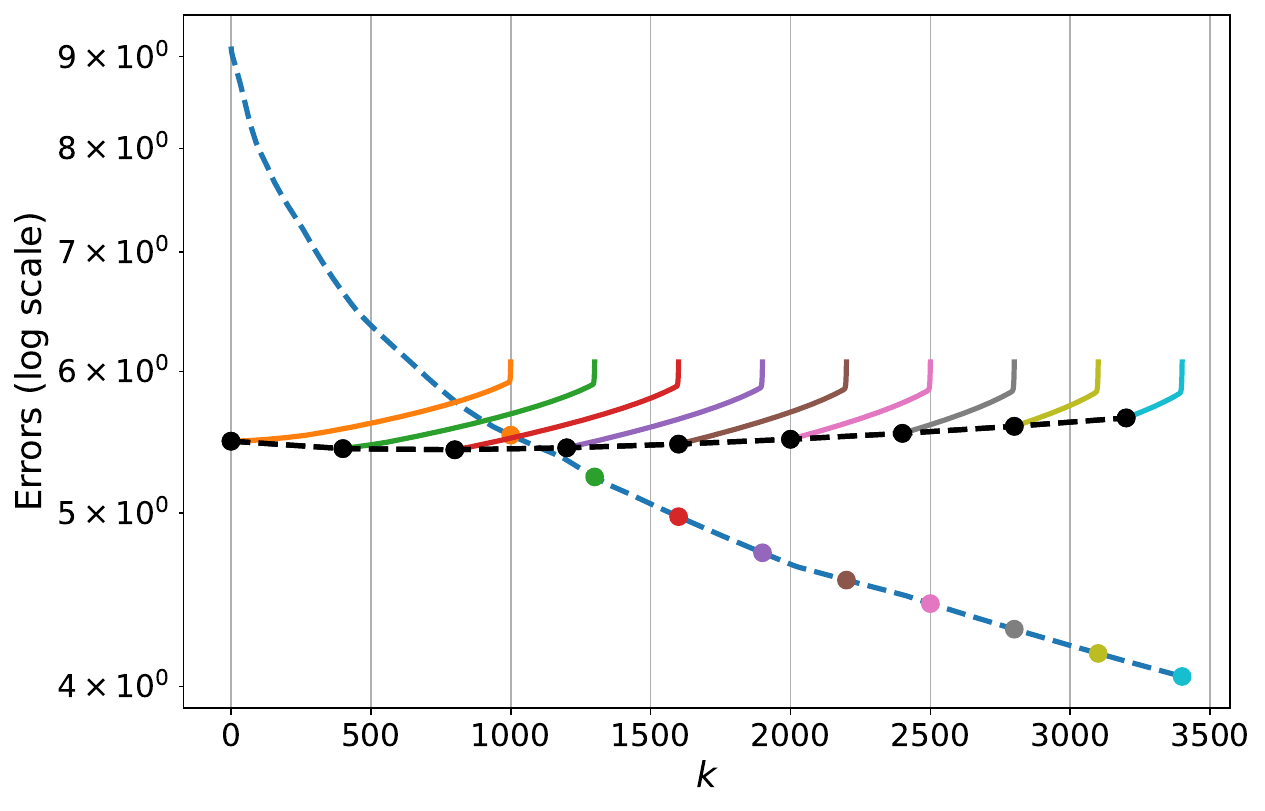}
    \caption{Reverse mode.}
    \label{subfig:latestart-dim40-subopt-bw}
  \end{subfigure}
  \caption{Late-start / truncation behavior for $N = 40$, $\alpha = 1 / (3L)$, and $\rho \approx \fpeval{round(\ratearray[1,2], 6)}$.}
  \label{fig:latestart-dim40-subopt}
\end{figure}
\clearpage

\section*{Acknowledgments}
Sheheryar Mehmood and Peter Ochs are supported by the German Research Foundation
(DFG Grant OC 150/4-1).


{\small
\bibliographystyle{ieee}
\bibliography{main}
}

\end{document}